\documentclass{article} 

\PassOptionsToPackage{sort}{natbib}

\usepackage[utf8]{inputenc} 
\usepackage[T1]{fontenc}    
\usepackage{hyperref}       
\usepackage{url}            
\usepackage{booktabs}       
\usepackage{amsfonts}       
\usepackage{amsmath}        
\usepackage{amsthm}         
\usepackage{amssymb}        
\usepackage{mathtools}      
\usepackage{nicefrac}       
\usepackage{microtype}      
\usepackage{graphicx}
\usepackage{dsfont}         
\usepackage{bm}             
\usepackage{mleftright}     
\usepackage{comment}        
\usepackage{xspace}         
\usepackage{paralist}       
\usepackage{subcaption}     
\usepackage{chngcntr}       
\usepackage{multirow}       
\usepackage{tikz}
\usepackage{pgfplots}
\usepackage{anyfontsize}    
\usepackage{adjustbox} 		  
\usepackage{rotating}

\usepackage[accepted]{icml2020}

\usetikzlibrary{arrows}
\usetikzlibrary{calc}
\usetikzlibrary{chains}
\usetikzlibrary{positioning}

\usepgfplotslibrary{colorbrewer}
\usepgfplotslibrary{dateplot}

\pgfplotsset{compat=1.16}

\bibpunct{\nolinebreak{}(}{)}{,}{a}{,}{,}

\DeclareMathOperator\supp{supp}                            
\DeclareMathOperator{\EE}        {\mathds{E}}              
\DeclareMathOperator{\diagram}   {\mathcal{D}}             
\DeclareMathOperator{\diameter}  {diam}                    
\DeclareMathOperator{\prob}      {\mathds{P}}              
\DeclareMathOperator{\dist}      {dist}                    
\DeclareMathOperator{\db}        {d_b}                     
\DeclareMathOperator{\dhd}       {d_H}                     
\DeclareMathOperator{\dgh}       {d_{GH}}                  
\DeclareMathOperator{\dkl}       {KL}                      
\DeclareMathOperator{\im}        {im}                      
\DeclareMathOperator{\vr}        {\mathfrak{R}}            
\DeclareMathOperator{\PH}        {PH}                      
\DeclareMathOperator{\loss}      {\mathcal{L}}             
\DeclareMathOperator{\density}   {f_{\sigma}}              
\DeclareMathOperator{\rank}      {rank}                    
\DeclareMathOperator{\persistence}{pers}                   

\newcommand{\betti}[1]           {\ensuremath{\beta_{#1}}} 
\newcommand{\manifold}           {\mathds{M}}              
\newcommand{\pointcloud}[1][X]   {#1}                      
\newcommand{\real}               {\mathds{R}}              
\renewcommand{\natural}          {\mathds{N}}              
\newcommand{\simplicialcomplex}  {\mathfrak{K}}            
\newcommand{\homologygroup}[1]   {\mathrm{H}_{#1}}         
\newcommand{\dspace}             {\mathcal{X}}             
\newcommand{\lspace}             {\mathcal{Z}}             
\newcommand{\mat}[1]             {\mathbf{#1}\!}           
\newcommand{\Index}              {\pi}                     

\newcommand{\boundary}[1]        {\ensuremath{\partial_{#1}}}

\newcommand{\chaingroup}[1]      {\ensuremath{C_{#1}}}

\newcommand{\persistentbetti}        [2]{\ensuremath{\betti{#1}^{#2}}}
\newcommand{\persistenthomologygroup}[2]{\ensuremath{\homologygroup{#1}^{#2}}}

\newtheorem{theorem}    {Theorem}

\newtheorem{observation}{Observation}

\newcommand{\bth}{\bm{\theta} }

\newcommand{\brho}{\bm{\rho}}

\newcommand\second[1]{\textbf{#1}} 				
\newcommand\first[1]{\textbf{\underline{#1}}}

\newcommand{\data}[1]{\mbox{\textsc{#1}}}   

\newcommand  {\st}{\textsuperscript{\textup{st}}\xspace}

\newcommand  {\nd}{\textsuperscript{\textup{nd}}\xspace}
\renewcommand{\th}{\textsuperscript{\textup{th}}\xspace}

\title{Topological Autoencoders}

\begin{document}

\twocolumn[
\icmltitle{Topological Autoencoders}



\icmlsetsymbol{equal}{$\dagger$}
\icmlsetsymbol{equal_last}{$\ddagger$}

\begin{icmlauthorlist}
\icmlauthor{Michael Moor}{equal,ethz,sib}
\icmlauthor{Max Horn}{equal,ethz,sib}
\icmlauthor{Bastian Rieck}{equal_last,ethz,sib}
\icmlauthor{Karsten Borgwardt}{equal_last,ethz,sib}
\end{icmlauthorlist}
\icmlaffiliation{ethz}{Department of Biosystems Science and Engineering, ETH Zurich, 4058 Basel, Switzerland}
\icmlaffiliation{sib}{SIB Swiss Institute of Bioinformatics, Switzerland}
\icmlcorrespondingauthor{Karsten Borgwardt}{karsten.borgwardt@bsse.ethz.ch}

\icmlkeywords{Machine Learning, ICML}

\vskip 0.3in
]



\newcommand{\EqualContribution}{\textsuperscript{$\dagger$}Equal contribution. \textsuperscript{$\ddagger$}These authors jointly directed this work.}
\printAffiliationsAndNotice{\EqualContribution} 


\begin{abstract}
    We propose a novel approach for preserving topological structures of the input space in latent representations of autoencoders.
    Using \emph{persistent homology}, a technique from topological data analysis, we calculate topological signatures of both the input
    and latent space to derive a topological loss term.
    Under weak theoretical assumptions, we construct this loss
    in a differentiable manner, such that the encoding learns to
    retain multi-scale connectivity information.
     We show that our approach is theoretically well-founded and
    that it exhibits favourable latent representations on a synthetic manifold
    as well as on real-world image data sets, while preserving low reconstruction errors.
\end{abstract}

\section{Introduction}

While topological features, in particular multi-scale features derived from persistent
homology, have seen increasing use in the machine learning
community~\citep{Reininghaus15, Hofer17,  Guss18, Carriere19, Hofer19a, Hofer19b,
Ramamurthy19, Rieck19a, Rieck19b}, employing topology \emph{directly} as a constraint
for modern deep learning methods remains a challenge.
This is due to the inherently discrete nature of these computations, making backpropagation through
the computation of topological signatures immensely difficult or only possible in certain special
circumstances~\citep{Poulenard18, Chen19, Hofer19a}.

This work presents a novel approach that permits obtaining
gradients during the computation of topological signatures. This makes
it possible to employ topological constraints while training deep neural
networks, as well as building topology-preserving autoencoders.
Specifically, we make the following contributions:
\begin{compactenum}
  \item We develop a new topological loss term for autoencoders that helps harmonise the topology
  of the data space with the topology of the latent space.
  \item We prove that our approach is stable on the level of mini-batches, resulting in suitable
  approximations of the persistent homology of a data set.
  \item We empirically demonstrate that our loss term aids in dimensionality reduction by
  preserving topological structures in data sets; in particular, the learned latent representations
  are useful in that the preservation of topological structures can improve interpretability.
\end{compactenum}

\begin{figure*}[tbp]
  \centering
  \subcaptionbox{$\epsilon_0$}{%
    \includegraphics[height=2.8cm]{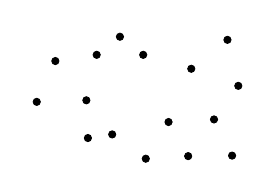}%
  }%
  \quad
  \subcaptionbox{$\epsilon_1$}{%
    \includegraphics[height=2.8cm]{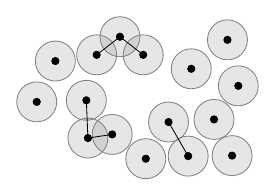}%
  }%
  \quad
  \subcaptionbox{$\epsilon_2$}{%
    \includegraphics[height=2.8cm]{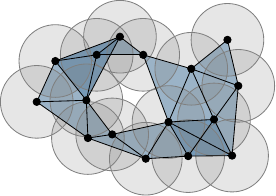}%
  }%
  \quad
  \subcaptionbox{$\diagram_{d}$\label{sfig:Persistence diagram example}}{%
    \includegraphics[height=2.8cm]{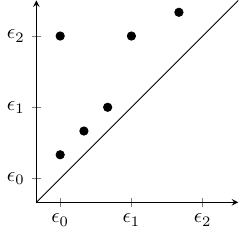}%
  }
  \caption{%
    The Vietoris--Rips complex $\vr_{\epsilon}\mleft(\pointcloud\mright)$ of a point cloud
    $\pointcloud$ at different scales $\epsilon_0$, $\epsilon_1$, and
    $\epsilon_2$.
    As the distance threshold $\epsilon$ increases, the connectivity
    changes. The creation and destruction of $d$-dimensional topological
    features is recorded in the $d$\th persistence diagram~$\diagram_d$.
  }
  \label{fig:Vietoris--Rips complex}
\end{figure*}

\section{Background: Persistent Homology}\label{sec:Background}

Persistent homology~\citep{Edelsbrunner08a, Barannikov94} is a method
from the field of computational topology, which develops tools for
analysing topological features~(connectivity-based features such as
connected components) of data sets.
We first introduce the underlying concept of simplicial homology.
For a simplicial complex $\simplicialcomplex$, i.e.\
a generalised graph with higher-order connectivity information such as
cliques, simplicial homology employs matrix reduction algorithms to
assign $\simplicialcomplex$ a family of groups, the
\emph{homology groups}.
The $d$\th homology group
$\homologygroup{d}\mleft(\simplicialcomplex\mright)$ of
$\simplicialcomplex$ contains \mbox{$d$-dimensional} topological features, such
as connected components~($d = 0$), cycles/tunnels~($d = 1$), and
voids~($d = 2$). Homology groups are typically summarised by their
ranks, thereby obtaining a simple invariant ``signature'' of a manifold.
For example, a circle in $\real^2$ has one feature with $d = 1$~(a cycle), and one feature with $d
= 0$~(a connected component).

In practice, the underlying manifold $\manifold$ is unknown and we are
working with a point cloud $\pointcloud := \mleft\{x_1, \dots,
x_n\mright\} \subseteq \real^d$ and a metric
$\dist\colon\pointcloud \times \pointcloud \to \real$ such as the
Euclidean distance.
Persistent homology extends simplicial homology to this setting: instead of
approximating $\manifold$ by means of a \emph{single} simplicial complex, which
would be an unstable procedure due to the discrete nature of $\pointcloud$, persistent
homology tracks changes in the homology groups over \emph{multiple} scales of the metric.
This is achieved by constructing a special simplicial complex, the Vietoris--Rips
complex~\citep{Vietoris27}.
For $0 \leq \epsilon < \infty$, the Vietoris--Rips complex of
$\pointcloud$ at scale $\epsilon$, denoted by
$\vr_{\epsilon}\mleft(\pointcloud\mright)$, contains all
simplices~(i.e.\ subsets) of $\pointcloud$ whose elements $\{x_0, x_1,
\dots\}$ satisfy $\dist\mleft(x_i,
x_j\mright) \leq \epsilon$ for all $i$, $j$.
Given a matrix $\mat{A}$ of pairwise distances of a point cloud
$\pointcloud$, we will use $\vr_{\epsilon}\mleft(\mat{A}\mright)$ and
$\vr_{\epsilon}\mleft(\pointcloud\mright)$ interchangeably because
constructing $\vr_{\epsilon}$ only requires distances.
Vietoris--Rips complexes satisfy a nesting relation, i.e.\
$\vr_{\epsilon_i}\mleft(\pointcloud\mright) \subseteq
\vr_{\epsilon_j}\mleft(\pointcloud\mright)$ for $\epsilon_i \leq
\epsilon_j$, making it possible to track changes in the homology groups
as $\epsilon$ increases~\citep{Edelsbrunner02}.
Figure~\ref{fig:Vietoris--Rips complex} illustrates this process.
Since $\pointcloud$ contains a finite number of points,
a maximum $\tilde{\epsilon}$ value exists for which the connectivity
stabilises; therefore, calculating $\vr_{\tilde{\epsilon}}$ is
sufficient to obtain topological features at all scales.

We write $\PH\mleft(\vr_{\epsilon}\mleft(\pointcloud\mright)\mright)$ for the persistent homology
calculation of the Vietoris--Rips complex. It results in a tuple $\mleft(\mleft\{\diagram_1,
\diagram_2, \dots\mright\}, \mleft\{\Index_1, \Index_2, \dots\mright\}\mright)$ of
\emph{persistence diagrams}~($1$\st component) and \emph{persistence
pairings}~($2$\nd component).
The $d$-dimensional persistence diagram
$\diagram_d$~(Figure~\ref{sfig:Persistence diagram example}) of
$\vr_{\epsilon}\mleft(\pointcloud\mright)$ contains coordinates of the
form $(a, b)$, where $a$ refers to a threshold $\epsilon$ at which
a $d$-dimensional topological feature is created in the Vietoris--Rips
complex, and $b$ refers to a threshold $\epsilon'$ at which it is
destroyed~(please refer to Supplementary Section~\ref{sec:Persistent
homology details} for a detailed explanation).
When $d = 0$, for example, the threshold $\epsilon'$ indicates at which
distance two connected components in $\pointcloud$ are merged into
one. This calculation is known to be related to spanning
trees~\citep{Kurlin15} and single-linkage clustering, but the
persistence diagrams and the persistence pairings carry more
information than either one of these concepts.

The $d$-dimensional persistence pairing contains indices $(i, j)$ corresponding to
simplices $s_i, s_j \in \vr_{\epsilon}\mleft(\pointcloud\mright)$ that create and destroy
the topological feature identified by $(a, b) \in \diagram_d$, respectively.
Persistence diagrams are known to be stable with respect to small
perturbations in the data set~\citep{Cohen-Steiner07}. Two diagrams
$\diagram$ and $\diagram'$ can be compared using the \emph{bottleneck
distance}
$
  \db(\diagram, \diagram') := \inf_{\eta\colon \diagram \to \diagram'}\sup_{x\in{}\diagram}\|x-\eta(x)\|_\infty
$,
where $\eta\colon\diagram \to \diagram'$ denotes a bijection between the points of the
two diagrams, and $\|\cdot\|_\infty$ refers to the $\mathrm{L}_\infty$ norm.
We use $\diagram^{\pointcloud}$ to refer to the \emph{set} of
persistence diagrams of a point cloud $\pointcloud$ arising from
$\PH\mleft(\vr_{\epsilon}\mleft(\pointcloud\mright)\mright)$.

\begin{figure}[tbp]
  \centering
  \definecolor{cardinal}{RGB}{196, 30, 58}
  \definecolor{yale}    {RGB}{ 70,130,180}

  \def\layersep{0.6cm}
  \tikzstyle{na} = [baseline=-.5ex]
  \colorlet{reconst}{cardinal}
  \colorlet{topo}{yale}
  \scalebox{0.80}{%
  \begin{tikzpicture}[shorten >=1pt, node distance=\layersep]
  \tikzstyle{neuron}=[circle,minimum size=3pt,inner sep=0pt, fill=black]
  \tikzstyle{con}=[gray, very thin]
  \tikzstyle{input neuron}=[neuron];
  \tikzstyle{output neuron}=[neuron, fill=black];
  \tikzstyle{hidden neuron}=[neuron, fill=black];
  \tikzstyle{annot} = [text width=5em, text centered]
  \tikzstyle{loss} = [text width=3.5cm, minimum height=0.75cm, align=center]

  \foreach \name / \y in {0,...,7}
      \path[yshift=-0.75cm]
          node[input neuron] (I-\name) at ($ (\layersep,-0.2*\y cm) $) {};
  \foreach \name / \y in {0,...,3}
      \path[yshift=-1.125cm]
          node[hidden neuron] (E-\name) at ($ (2*\layersep,-0.2*\y cm) $) {};
  \foreach \name / \y in {0,...,1}
      \path[yshift=-1.325cm]
          node[output neuron] (Z-\name) at ($ (3*\layersep,-0.2*\y cm) $) {};

  \foreach \name / \y in {0,...,3}
      \path[yshift=-1.125cm]
          node[hidden neuron] (D-\name) at ($ (4*\layersep,-0.2*\y cm) $) {};
  \foreach \name / \y in {0,...,7}
      \path[yshift=-0.75cm]
          node[input neuron] (O-\name) at ($ (5*\layersep,-0.2*\y cm) $) {};

  \foreach \source in {0,...,7}
     \foreach \dest in {0,...,3}
         \path[con] (I-\source) edge (E-\dest);


  \foreach \source in {0,...,3}
      \foreach \dest in {0,...,1}
          \path[con] (E-\source) edge (Z-\dest);

  \foreach \source in {0,...,1}
      \foreach \dest in {0,...,3}
          \path[con] (Z-\source) edge (D-\dest);

  \foreach \source in {0,...,3}
      \foreach \dest in {0,...,7}
          \path[con] (D-\source) edge (O-\dest);


  \node[annot,below of=Z-1, node distance=1cm] (latent) {\large $\pointcloud[Z]$ \\ \small Latent code};
  \node[annot,left of=I-3, node distance=0.25cm, anchor=east] (input)
  {\large$\pointcloud$ \\ \small Input data};
  \node[annot,right of=O-3, node distance=0.25cm, anchor=west] (reconst)
  {\large $\tilde{\pointcloud}$ \\ \small Reconstruction};

  \node[loss, right of=reconst, node distance=0.8cm, anchor=west, yshift=-1.20cm] (loss-reconst) {
      Reconstruction loss
  };
  \coordinate (ref) at ($ (I-0.north) + (0,0.5) $);
  \draw[thick, reconst, -latex] (input.north) -- (input.north |- ref) -| (loss-reconst.north);
  \draw[thick, reconst] (reconst.north)--($ (reconst.north |- ref) + (0, 0.03cm) $);

  \coordinate (pers-diag-y) at ($ (I-7.south) - (0,0.5) $);

      \begin{axis}[%
          width         = 3cm,
          height        = 3cm,
          axis x line   = bottom,
          axis y line   = left,
          mark size     = 1pt,
          xmin          = -1.0,
          ymin          = -1.0,
          xmax          =  9.0,
          ymax          =  9.0,
          xlabel        = $\epsilon$,
          ylabel        = $\epsilon'$,
          ticks         = none,
          ylabel style  = {rotate=-90},
          at            = {($ (input.south |- pers-diag-y) + (0, -0.75cm)$)},
          anchor        = north,
          name          = {input-pers}
        ]
        \addplot[black,only marks] file {data/Persistence_diagram_example.txt};
        \addplot[domain=-1:9, inner sep=0pt] {x};
      \end{axis}

      \begin{axis}[%
          width         = 3cm,
          height        = 3cm,
          axis x line   = bottom,
          axis y line   = left,
          mark size     = 1pt,
          xmin          = -1.0,
          ymin          = -1.0,
          xmax          =  9.0,
          ymax          =  9.0,
          xlabel        = $\epsilon$,
          ylabel        = $\epsilon'$,
          ticks         = none,
          ylabel style  = {rotate=-90},
          at            = {($ (latent.south |- pers-diag-y) + (0, -0.75cm)$)},
          anchor        = north,
          name          = {latent-pers}
        ]
        \addplot[black,only marks] file {data/Persistence_diagram_example.txt};
        \addplot[domain=-1:9, inner sep=0pt] {x};
      \end{axis}

      \draw[thick, topo, -latex] (input.south) -- (input-pers.north);
      \draw[thick, topo, -latex] (latent.south) -- (latent-pers.north);

      \node[loss, below=0.50cm of loss-reconst, anchor=north] (loss-topo) {
          Topological loss
      };

      \coordinate (arrow-topo-y) at ($ (input-pers.south) + (0, -0.75cm) $);

      \draw[thick, topo, -latex]
         ($ (input-pers.south) + (0, -0.4cm) $) --
         (input-pers.south |- arrow-topo-y) -|
         (loss-topo.south);
      \draw[thick, topo]
          ($ (latent-pers.south) + (0, -0.4cm) $) --
          ($ (latent-pers.south |- arrow-topo-y) - (0, 0.03cm) $);
  \end{tikzpicture}
  }%
  \caption{%
    An overview of our method. Given a mini-batch $\pointcloud$ of data
    space $\dspace$, we train an autoencoder to reconstruct
    $\pointcloud$, leading to a reconstruction $\tilde{\pointcloud}$.
    In addition to the usual reconstruction loss, we calculate
    our \emph{topological loss} based on the topological differences
    between persistence diagrams, i.e.\ topological feature
    descriptors, calculated on the mini-batch~$\pointcloud$ and its
    corresponding latent code~$\pointcloud[Z]$. The objective of our
    topological loss term is to constrain the autoencoder such that
    topological features in the data space are preserved in latent
    representations.
  }
  \label{fig:Overview}
\end{figure}
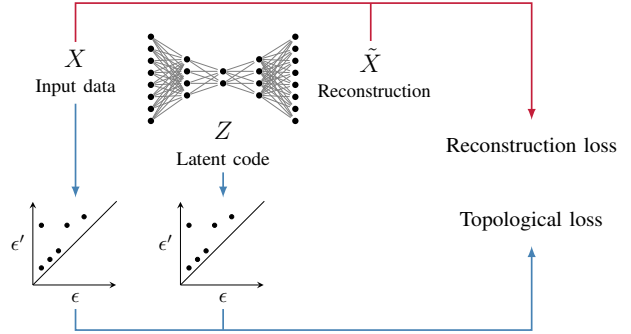

\section{A Topology-Preserving Autoencoder}

We propose a generic framework for constraining autoencoders to preserve topological
structures~(measured via persistent homology) of the data space in their latent encodings.
Figure~\ref{fig:Overview} depicts an overview of our method; the
subsequent sections will provide more details about the individual
steps. 

\subsection{Vietoris--Rips Complex Calculation}\label{sec:Vietoris--Rips complex}

Given a finite metric space $\mathcal{S}$, such as a point cloud, we first calculate the
persistent homology of the Vietoris--Rips complex of its distance matrix
$\mat{A}^{\mathcal{S}}$. It is common practice to use the Euclidean
distance for the calculation of $\mat{A}^{\mathcal{S}}$, but both the persistent homology
calculation and our method are \emph{not} restricted to any particular
distance; previous research~\citep{Wagner14} shows that even similarity
measures that do not satisfy the properties of a metric can be used
successfully with $\PH(\cdot)$.
Subsequently, let $\epsilon := \max\mat{A}^{\mathcal{S}}$ so
that $\vr_{\epsilon}\mleft(\mathcal{\mat{A}^{\mathcal{S}}}\mright)$ is
the corresponding Vietoris--Rips complex as described in
Section~\ref{sec:Background}.
Given a maximum dimension\footnote{%
  This means that we do not have to consider higher-dimensional
topological features, making the calculation more efficient.
}
of $d \in \natural_{> 0}$, we obtain a set of persistence
diagrams $\diagram^{\mathcal{S}}$, and a set of persistence
pairings $\Index^\mathcal{S}$.
The $d$\th persistence pairing $\Index_d^{\mathcal{S}}$
contains indices of simplices that are pertinent to the creation and
destruction of $d$-dimensional topological features.
We can consider each pairing to represent \emph{edge indices}, namely
the edges that are deemed to be ``topologically relevant'' by the
computation of persistent homology~(see below for more details).
This works because the Vietoris--Rips complex is a \emph{clique
complex}, i.e.\ a simplicial complex that is fully determined by
its edges~\citep{Zomorodian10a}.

\paragraph{Selecting indices from pairings}
The basic idea of our method involves selecting indices in the
persistence pairing and mapping them back to a distance between two
vertices. We then adjust this distance to harmonise topological features
of the input space and the latent space.
For \mbox{$0$-dimensional} topological features, it is sufficient to consider
the indices of \emph{edges}, which are the ``destroyer'' simplices,
in the pairing $\Index^\mathcal{S}_0$.
Each index corresponds to an edge in the minimum spanning tree of the data set. This
calculation is computationally efficient, having a worst-case complexity of $\mathcal{O}\mleft(m^2
\cdot \alpha\mleft(m^2\mright)\mright)$, where $m$ is the batch size and $\alpha(\cdot)$ denotes the
extremely slow-growing inverse Ackermann function~\citep[Chapter~22]{Cormen09}.
For \mbox{$1$-dimensional} features, where edges are paired with
triangles, we obtain edge indices by selecting the edge with the maximum
weight of the triangle. While this procedure, and thus our method,
generalises to higher dimensions, our current implementation supports no
higher-dimensional features.
Since preliminary experiments showed that using $1$-dimensional
topological features merely increases runtime, the subsequent
experiments will focus only on \mbox{$0$-dimensional} persistence diagrams.
We thus use $\mleft(\diagram^\mathcal{S}, \Index^\mathcal{S}\mright)$ to
denote the \mbox{0-dimensional} persistence diagram and pairing of
$\mathcal{S}$, respectively.

\subsection{Topological Autoencoder}
\label{Sec:topoAE}

In the following, we consider a mini-batch $\pointcloud$ of size $m$
from the data space $\dspace$ as a point cloud. Furthermore, we define
an autoencoder as the composition of two functions $h \circ g$, where
$g\colon \dspace \to \lspace$ represents the \emph{encoder} and $h\colon
\lspace \to \dspace$ represents the \emph{decoder}, denoting latent
codes by $\pointcloud[Z] := g\mleft( \pointcloud \mright)$.
During a forward pass of the autoencoder, we compute the persistent homology of the mini-batch in
both the data as well as the latent space, yielding two sets of tuples,
i.e.\
$\mleft( \diagram^{\pointcloud}, \Index^{\pointcloud} \mright) := \PH\mleft(\vr_{\epsilon}\mleft(\pointcloud\mright)\mright)$
and $\mleft( \diagram^{\pointcloud[Z]}, \Index^{\pointcloud[Z]} \mright) := \PH\mleft(\vr_{\epsilon}\mleft(\pointcloud[Z]\mright)\mright)$.
The values of the persistence diagram can be retrieved by subsetting the
distance matrix with the edge indices provided by the persistence
pairings; we write
$\diagram^{\pointcloud} \simeq \mat{A}^{\pointcloud}\mleft[\Index^{\pointcloud} \mright]$
to indicate that the diagram, which is a set, contains the same
information as the distances we retrieve with the pairing.
We treat $\mat{A}^{\pointcloud}\mleft[ \Index^{\pointcloud}\mright]$ as
a vector in $\mathds{R}^{|\Index^{\pointcloud}|}$.
Informally speaking, the persistent homology calculation can thus be seen as a selection of
topologically relevant edges of the Vietoris--Rips complex, followed by the selection of
corresponding entries in the distance matrix.
By comparing both diagrams, we can construct a topological regularisation term
$\loss_t := \loss_{t}\mleft(\mat{A}^{\pointcloud}, \mat{A}^{\pointcloud[Z]}, \Index^{\pointcloud},
 \Index^{\pointcloud[Z]}\mright)$, which we add to the
reconstruction loss of an autoencoder, i.e.\
\begin{align}
    \loss :=        \loss_{r}\mleft(\pointcloud, h\mleft(g\mleft(\pointcloud\mright)\mright)\mright)
          + \lambda \loss_{t}
\end{align}
where $\lambda \in \real$ is a parameter to control the strength of the
regularisation~(see also Supplementary Section~\ref{Supp:Arch and Hyper}).

Next, we discuss how to specify $\loss_t$. Since we only select edge
indices from $\Index^{\pointcloud}$ and $\Index^{\pointcloud[Z]}$, the
$\PH$ calculation represents a selection of topologically relevant
\emph{distances} from the distance matrix.
Each persistence diagram entry corresponds to a distance between two
data points. Following standard assumptions in persistent
homology~\citep{Poulenard18, Hofer19a}, we assume that the distances are
\emph{unique} so that each entry in the diagram has an infinitesimal
neighbourhood that only contains a single point.
In practice, this can always be achieved by performing~(symbolic)
perturbations of the distances.
Given this fixed pairing and a differentiable distance function, the
persistence diagram entries are therefore \emph{also} differentiable
with respect to the encoder parameters.
Hence, the persistence pairing does not change upon a small perturbation
of the underlying distances, thereby guaranteeing the existence of the
derivative of our loss function. This, in turn, permits the calculation
of gradients for backpropagation.

A straightforward approach to impose the data space topology on the
latent space would be to directly calculate a loss based on the selected
distances in both spaces.
Such an approach will \emph{not} result in informative gradients for
the autoencoder, as it merely compares topological features without
matching\footnote{%
  We use the term ``matching'' only to build intuition. Our approach
  does not calculate a matching in the sense of a bottleneck or
  Wasserstein distance between persistence diagrams.
}
the edges between $\vr_{\epsilon}\mleft(\pointcloud \mright)$
and $\vr_{\epsilon}\mleft(\pointcloud[Z] \mright)$.
A cleaner approach would be to enforce similarity on the intersection of
the selected edges in both complexes. However, this would initially
include very few edges, preventing efficient training and leading to
highly biased estimates of the topological alignments between the
spaces\footnote{%
  When initialising a random latent space $\pointcloud[Z]$, the
  persistence pairing in the latent space will select random
  edges, resulting in only $1$ expected matched edge~(independent of
  mini-batch size) between the two pairings. Thus, only one
  edge~(referring to one pairwise distance between two latent codes)
  could be used to update the encoding of these two data points.
}.
To overcome this, we account for the \emph{union} of all selected edges in
$\pointcloud$ and $\pointcloud[Z]$.
Our topological loss term decomposes into two components, each handling
the ``directed'' loss occurring as topological features in one of the
two spaces remain fixed. Hence,
$\loss_{t} = \loss_{\dspace \to \lspace} + \loss_{\lspace \to \dspace}$,
with
\begin{align}
  \loss_{\dspace \to \lspace} := \frac{1}{2} \mleft\| \mat{A}^{\pointcloud}\mleft[ \Index^{\pointcloud} \mright] - \mat{A}^{Z}\mleft[ \Index^{\pointcloud} \mright] \mright\|^2\\
  \shortintertext{and}
  \loss_{\lspace \to \dspace} := \frac{1}{2} \mleft\| \mat{A}^{Z}\mleft[ \Index^{Z} \mright] - \mat{A}^{\pointcloud}\mleft[ \Index^{Z} \mright] \mright\|^2,
\end{align}
respectively. The key idea for both terms is to align and preserve
topologically relevant distances from both spaces. By taking the union
of all selected edges~(and the corresponding distances), we obtain an
informative loss term that is determined by at least $|\pointcloud|$
distances.
This loss can be seen as a more generic version of the loss introduced by
\citet{Hofer19a}, whose formulation does not take the two directed
components into account and optimises the destruction values of all
persistence tuples with respect to a uniform parameter~(their goal
is different from ours and does not require a loss term
that is capable of harmonising topological features across the two
spaces; please refer to Section~\ref{sec:Related
work} for a brief discussion).
By contrast, our formulation aims to
to align the distances between $\pointcloud$ and
$\pointcloud[Z]$~(which in turn will lead to an alignment of distances
between $\dspace$ and $\lspace$).
If the two spaces are aligned perfectly, $\loss_{\dspace \to \lspace}
= \loss_{\lspace \to \dspace} = 0$ because both pairings and their
corresponding distances coincide.
The converse implication is not true: if $\loss_{t} = 0$, the
persistence pairings and their corresponding persistence diagrams are
not necessarily identical. Since we did not observe such behaviour in
our experiments, however, we leave a more formal treatment of these situations
for future work.

\paragraph{
  Gradient calculation
}%
Letting $\bth$ refer to the parameters of the \emph{encoder}
and using $\brho := \mleft(\mat{A}^{\pointcloud}\mleft[ \Index^{\pointcloud}
\mright] - \mat{A}^{Z}\mleft[ \Index^{\pointcloud} \mright]\mright)$,
we have
\begin{align}
  \frac{\partial}{\partial\bth} \loss_{\dspace \to \lspace} &=
    \frac{\partial}{\partial\bth} \mleft(\frac{1}{2}\mleft\| \mat{A}^{\pointcloud}\mleft[
      \Index^{\pointcloud} \mright] - \mat{A}^{Z}\mleft[ \Index^{\pointcloud} \mright] \mright\|^2
    \mright)\\
                                                            &= - \brho^\top \mleft(\frac{\partial\mat{A}^{Z}\mleft[ \Index^{\pointcloud} \mright]}{\partial\bth} \mright)\\
      &= - \brho^\top \mleft(\sum_{i = 1}^{\mleft|\Index^{\pointcloud}\mright|}
      \frac{\partial\mat{A}^{Z}\mleft[ \Index^{\pointcloud} \mright]_{i}}{\partial\bth} \mright),
\end{align}
where $\mleft|\Index^{\pointcloud}\mright|$ denotes the cardinality of a persistence pairing and
$\mat{A}^{Z}\mleft[ \Index^{\pointcloud} \mright]_{i}$ refers to the $i$\th entry of the vector of
paired distances. This derivation works analogously for $\loss_{\lspace \to \dspace}$~(with
$\Index^{\pointcloud}$ being replaced by $\Index^{\pointcloud[Z]}$).
Furthermore, any derivative of $\mat{A}^{\pointcloud}$ with respect to
$\bth$ must vanish because the distances of the input samples do not
depend on the encoding by definition.
These equations assume infinitesimal perturbations. The persistence
diagrams change in a non-differentiable manner during the training
phase. However, for any given update step, a diagram is robust to
infinitesimal changes of its entries~\citep{Cohen-Steiner07}. As
a consequence, our topological loss is differentiable for each update
step during training.
We make our code publicly available\footnote{%
  \url{https://github.com/BorgwardtLab/topological-autoencoders}
}.

\subsection{Stability}
\label{Sec:stability}

Despite the aforementioned known stability of persistence diagrams with
respect to small perturbations of the underlying space, we still have to
analyse our topological approximation on the level of mini-batches. The
following theorem guarantees that subsampled persistence diagrams are
close to the persistence diagrams of the original point cloud.
\begin{theorem}
  Let $\pointcloud$ be a point cloud of cardinality $n$ and $\pointcloud^{(m)}$ be one subsample of
  $\pointcloud$ of cardinality $m$, i.e.\ $\pointcloud^{(m)} \subseteq \pointcloud$, sampled without
  replacement. We can bound the probability of the persistence diagrams
  of $\pointcloud^{(m)}$ exceeding a threshold in terms of
  the bottleneck distance as
  \begin{equation}
    \prob\mleft(\db\!\mleft(\diagram^{\pointcloud}\!\!, \diagram^{\pointcloud^{(m)}}\mright)\!>\!\epsilon \mright)\leq \prob\mleft( \dhd\!\mleft(\pointcloud, \pointcloud^{(m)}\mright)\!>\!2\epsilon \mright),
    \label{eq:Probability bottleneck distance}
  \end{equation}
  where $\dhd$ refers to the Hausdorff distance between the point cloud
  and its subsample.
  \label{thm:Probability bottleneck distance}
\end{theorem}
\begin{proof}
  See Section~\ref{sec:Proof probability} in the supplementary materials.
\end{proof}
For $m \to n$, each mini-batch converges to the original point cloud, so
we have $\lim_{m \to n}\dhd\mleft(\pointcloud, \pointcloud^{(m)}\mright)
= 0$.
Please refer to Section~\ref{sec:Empirical convergence rates} for an
analysis of empirical convergence rates as well as a discussion of
a worst-case bound.
Given certain independence assumptions, the next theorem approximates
the expected value of the Hausdorff distance between the point cloud and
a mini-batch. The calculation of an exact representation is beyond the
scope of this work.
\begin{theorem}
  Let $\mat{A} \in \real^{n \times m}$ be the distance matrix between samples of $\pointcloud$
  and $\pointcloud^{(m)}$, where the rows are sorted such that the first $m$ rows correspond to the
  columns of the $m$ subsampled points with diagonal elements $a_{ii} = 0$. Assume that the entries
  $a_{ij}$ with $i>m$ are random samples following a distance distribution $F_D$ with $\supp(F_D) \in
  \real_{\geq 0}$.
  The minimal distances $\delta_{i}$ for rows with $i > m$ follow a distribution $F_{\Delta}$.
  Letting $Z := \max_{1 \leq i \leq n}\delta_{i}$ with a corresponding distribution $F_Z$, the
  expected Hausdorff distance between $\pointcloud$ and $\pointcloud^{(m)}$ for $m < n$ is bounded by
  \begin{align}
    \EE \mleft[ \dhd(\pointcloud, \pointcloud^{(m)}) \mright] &= \EE_{Z \sim F_Z} \mleft[ Z \mright]\\
    &\leq\!\!\int\displaylimits_{0}^{+\infty} \mleft( 1 - F_{\Delta}(z)^{n-m} \mright)
      \operatorname{d}\!z  \label{Eq:upper bound},
     \shortintertext{where}
    F_{\Delta}(z) &= - \sum_{k=1}^m {m \choose k}\mleft( - F_D\mleft(z\mright)\mright)^{m-k}.
  \end{align}
  \label{thm:ExpectedHausdorff}
\end{theorem}
\begin{proof}
  See Section~\ref{sec:Proof expectation} in the supplementary materials.
\end{proof}
From Eq.~\ref{Eq:upper bound}, we obtain $\EE\mleft[ \dhd(\pointcloud,
\pointcloud^{m}) \mright] = 0$ as $m \to n$,
so the expected value converges as the subsample size approaches
the total sample size\footnote{For $m = n$, the two integrals
switch their order as $m(n-m) = 0 < n - 1$ (for $n > 1$).}.
We conclude that our subsampling approach results in point clouds that are suitable proxies for the
large-scale topological structures of the point cloud $\pointcloud$.

\section{Related Work}\label{sec:Related work}

Computational topology and persistent homology~(PH) have started gaining
traction in several areas of machine learning research.
PH is often used as as \emph{post hoc} method for analysing topological characteristics of
data sets.
Thus, there are several methods that compare topological features of high-dimensional
spaces with different embeddings to assess the fidelity and quality of
a specific embedding scheme~\citep{Rieck15a, Rieck15b, Paul17, Khrulkov18, Yan18}.
PH can also be used to characterise the training of neural
networks~\citep{Guss18, Rieck19a}, as well as their decision
boundaries~\citep{Ramamurthy19}.
Our method differs from all these publications in that we are able to obtain gradient
information to \emph{update} a model while training.
Alternatively, topological features can be integrated into classifiers
to improve classification performance. \citet{Hofer17} propose a neural
network layer that learns projections of persistence diagrams,
which can subsequently be used as feature descriptors to classify
structured data.
Moreover, several vectorisation strategies for persistence diagrams
exist~\citep{Carriere15, Adams17}, making it possible to use them in
kernel-based classifiers. These strategies have
been subsumed~\citep{Carriere19} in a novel architecture based on
deep sets.
The commonality of these approaches is that they treat persistence
diagrams as being \emph{fixed}; while they are capable of learning
suitable parameters for classifying them, they cannot adjust input data
to better approximate a certain topology.

Such topology-based adjustments have only recently become feasible.
\citet{Poulenard18} demonstrated how to optimise real-valued
functions based on their topology. This constitutes the first approach
for aligning persistence diagrams by modifying input data; it requires
the connectivity of the data to be known, and the optimised functions
have to be node-based and scalar-valued.
By contrast, our method works \emph{directly} on distances and sidesteps
connectivity calculations via the Vietoris--Rips complex. \citet{Chen19}
use a similar optimisation technique to regularise the decision boundary
of a classifier. However, this requires discretising the space, which can
be computationally expensive.
\citet{Hofer19a}, the closest work to ours, also presents
a differentiable loss term. Their formulation enforces a \emph{single}
scale, referred to as~$\eta$, on the latent space. The learned encoding
is then applied to a one-class learning task in which a scoring function is
calculated based on the pre-defined scale. By contrast, the goal of our
loss term is to support the model in learning a latent encoding that
\emph{best preserves} the data space topology in said latent space,
which we use for dimensionality reduction. We thus target a different
task, and can preserve \emph{multiple} scales~(those selected through
the filtration process) that are present in the data domain.

\section{Experiments}

Our main task is to learn a latent space in an unsupervised manner such
that topological features of the data space, measured using persistent
homology approximations on every batch, are preserved as much as
possible.

\subsection{Experimental Setup}\label{sec:Setup}

Subsequently, we briefly describe our data sets and evaluation metrics.
Please refer to the supplementary materials for technical
details~(calculation, hyperparameters, etc.).

\subsubsection{Data Sets}
%
We generate a \data{Spheres} data set that consists of ten high-dimensional
\mbox{$100$-spheres} living in a \mbox{$101$-dimensional} space that are
enclosed by one larger sphere that consists of the same number of points
as the total of inner spheres~(see Section~\ref{Supp:Datasets} for more
details).
We also use three image data sets~(\data{MNIST}, \data{Fashion-MNIST},
and \data{CIFAR-10}), which are particularly amenable to our
topology-based analysis because real-world images are known to lie
\emph{on} or \emph{near} low-dimensional manifolds~\citep{Lee03,
Peyre09}.

\subsubsection{Baselines \& Training Procedure}
%
We compare our approach with several dimensionality reduction
techniques, including UMAP~\citep{mcinnes2018umap},
\mbox{t-SNE}~\citep{maaten2008visualizing},
Isomap~\citep{tenenbaum2000global}, PCA, as well as standard
autoencoders~(AE). We apply our proposed topological constraint to this
standard autoencoder architecture~(TopoAE).

For comparability and
interpretability, each method is restricted to two latent dimensions.
We split each data set into training and testing~(using the predefined
split if available; 90\% versus 10\% otherwise).
Additionally, we remove 15\% of the training split as a validation data
set for tuning the hyperparameters.
We normalised our topological loss term by the batch size $m$ in order
to disentangle $\lambda$ from it.
All autoencoders employ batch-norm and are optimised using
ADAM~\citep{kingma2014adam}. Since \mbox{t-SNE} is not intended to be
applied to previously unseen test samples, we evaluate this
method only on the train split.
In addition, significant computational scaling issues prevent us from
running a hyperparameter search for Isomap on real-world data sets, so
we only compare this algorithm on the synthetic data set.
Please refer to Section~\ref{Supp:Arch and Hyper} for more details on
architectures and hyperparameters.

\begin{figure}[t]
  \centering
  \subcaptionbox{PCA}{%
    \includegraphics[width=0.50\linewidth]{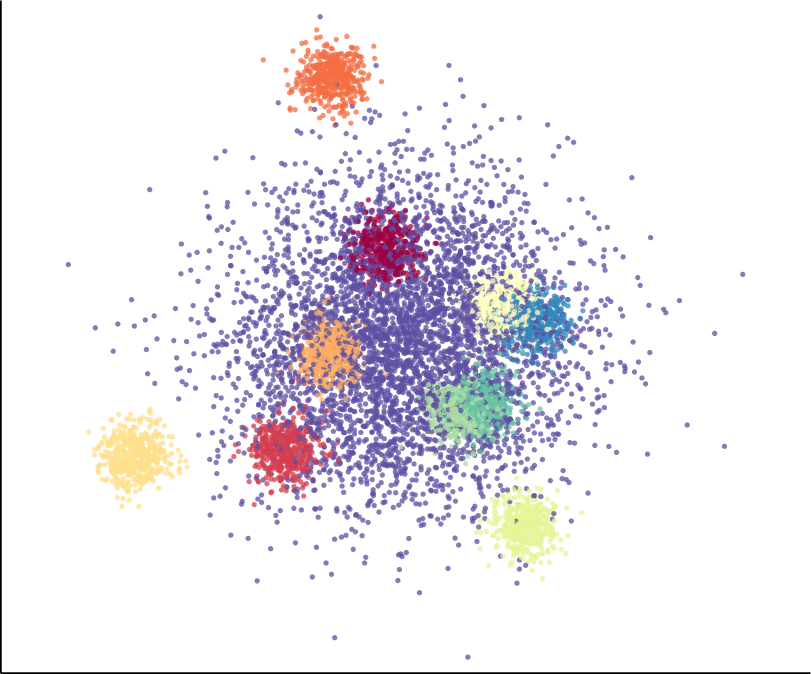}%
  }%
  \subcaptionbox{Isomap}{%
    \includegraphics[width=0.50\linewidth]{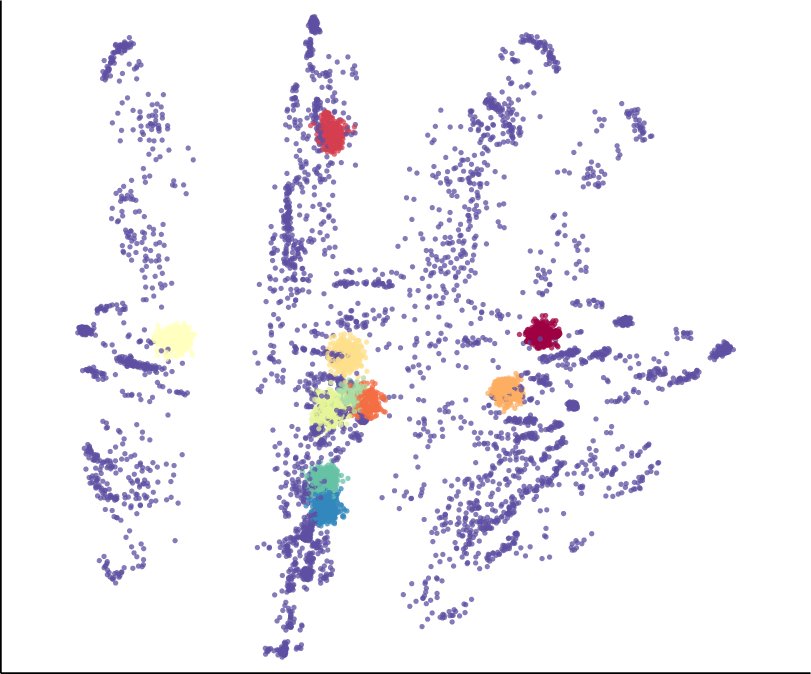}%
  }\\
  \subcaptionbox{t-SNE\label{sfig:Spheres t-SNE}}{%
    \includegraphics[width=0.50\linewidth]{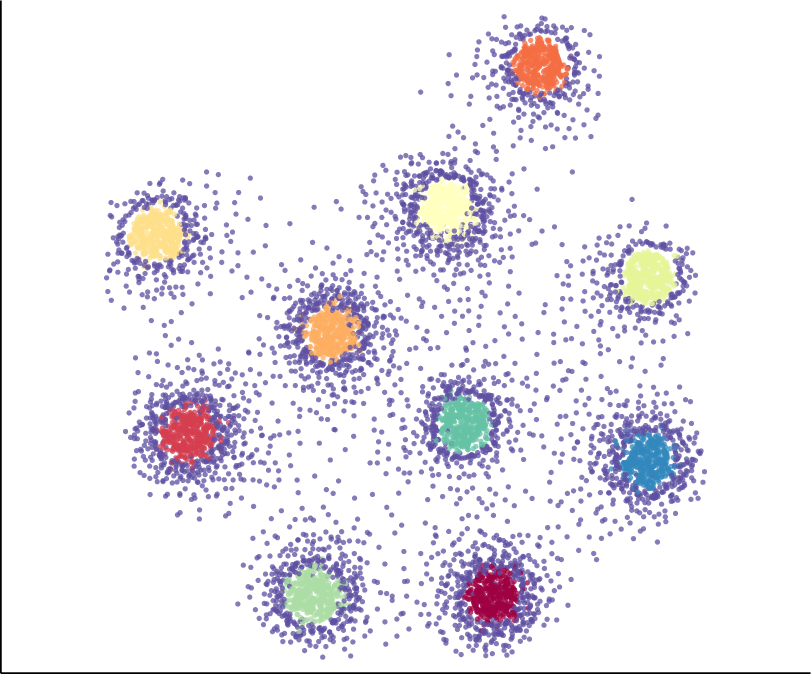}%
  }%
  \subcaptionbox{UMAP\label{sfig:Spheres UMAP}}{%
    \includegraphics[width=0.50\linewidth]{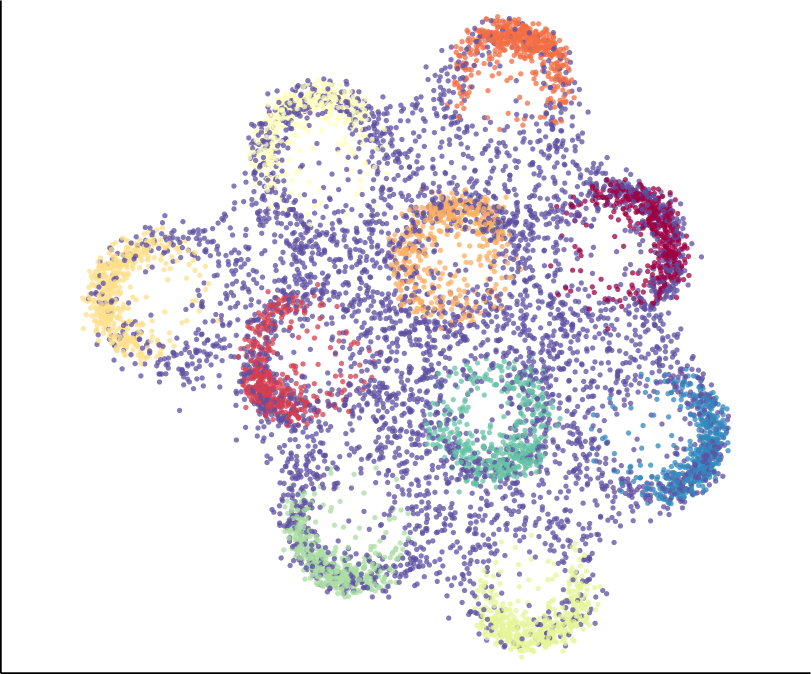}%
  }\\
  \subcaptionbox{AE}{%
    \includegraphics[width=0.50\linewidth]{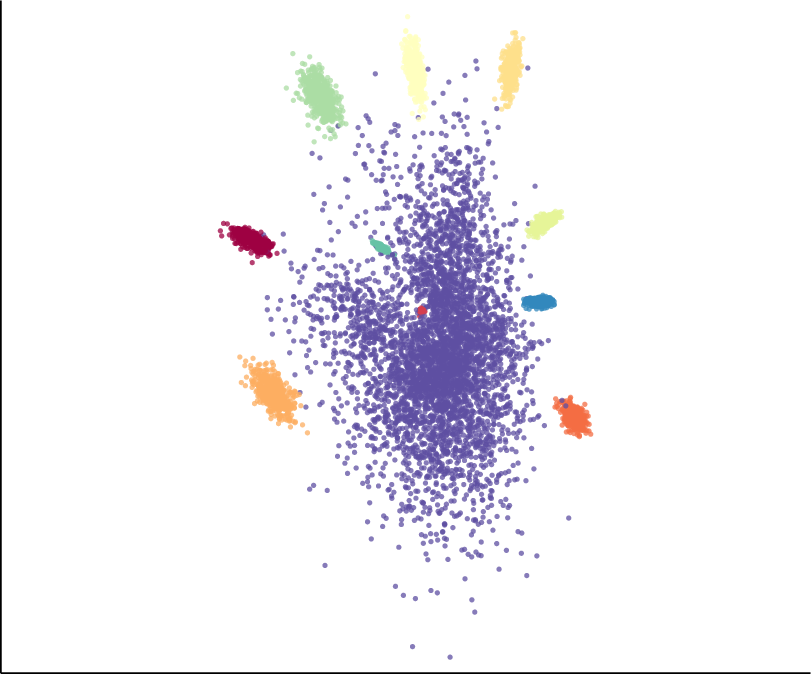}%
  }%
  \subcaptionbox{TopoAE\label{sfig:Spheres TAE}}{%
    \includegraphics[width=0.50\linewidth]{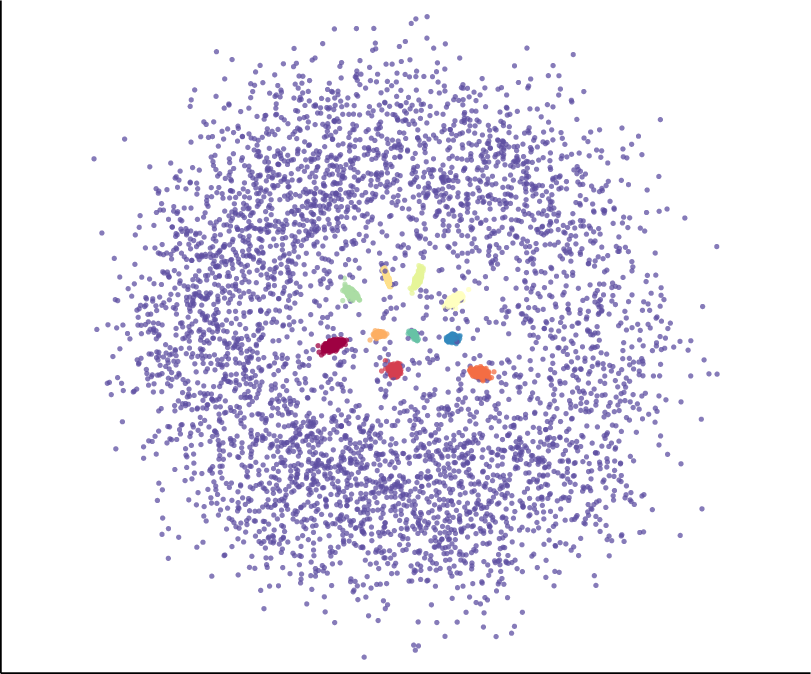}
  }%
  \caption{%
    Latent representations of the \data{Spheres} data set. Only our
    method is capable of representing the complicated nesting
    relationship inherent to the data; \mbox{t-SNE}, for example, tears
    the original data apart. For TopoAE, we used a batch size of $28$.
    Please refer to Figure~\ref{fig:Spheres all methods} in the supplementary materials for an enlarged version.
  }
  \label{fig:Spheres selected methods}
\end{figure}

\subsubsection{Evaluation}
%
We evaluate the quality of latent representations in terms of
\begin{inparaenum}[(1)]
  \item low-dimensional visualisations,
  \item dimensionality reduction quality metrics~(evaluated between
    input data and \emph{latent} codes), and
  \item reconstruction errors~(Data MSE; evaluated between input and
    \emph{reconstructed} data), provided that invertible transformations
    are available\footnote{Invertible transformations are available
    for PCA and all autoencoder-based methods.}.
\end{inparaenum}
For~(2), we consider several measures~(please refer to
Section~\ref{sec:Quality measures} for more details). First, we
calculate $\dkl_{\sigma}$, the Kullback--Leibler divergence between
the density estimates of the input and latent space, based on density
estimates~\citep{Chazal11, Chazal14a}, where $\sigma \in \real_{> 0}$
denotes the length scale of the Gaussian kernel, which is varied to
account for multiple data scales. We chose minimising $\dkl_{0.1}$ as
our hyperparameter search objective.
Furthermore, we calculate common non-linear dimensionality
reduction~(NLDR) quality metrics, which use the pairwise distance matrices of the input and
the \emph{latent} space~(as indicated by the ``$\ell$'' in the
abbreviations), namely
\begin{inparaenum}[(1)]
  \item the \emph{root mean square error}~($\ell$-RMSE), which---despite
    its name---is not related to the reconstruction error of the
    autoencoder but merely measures to what extent the two distributions
    of distances coincide,
  \item the \emph{mean relative rank error}~($\ell$-MRRE),
  \item the \emph{continuity}~($\ell$-Cont), and
  \item the \emph{trustworthiness}~($\ell$-Trust)
  .
\end{inparaenum}
The reported measures are computed on the test
splits~(except for \mbox{t-SNE} where no transformation between splits
is available, so we report the measures on a random subsample of the
train split, preserving the cardinality of the test split).

\begin{figure}[!ht]
  \centering

  \captionsetup[sub]{skip=3pt}

  \subcaptionbox*{\small\data{Fashion-MNIST}}{%
    \hspace*{0.33\linewidth}
  }%
  \subcaptionbox*{\small\data{MNIST}}{%
    \hspace*{0.33\linewidth}
  }%
  \subcaptionbox*{\small\data{CIFAR-10}}{%
    \hspace*{0.33\linewidth}
  }\\[0.25cm]
  \subcaptionbox*{}{%
    \includegraphics[width=0.33\linewidth]{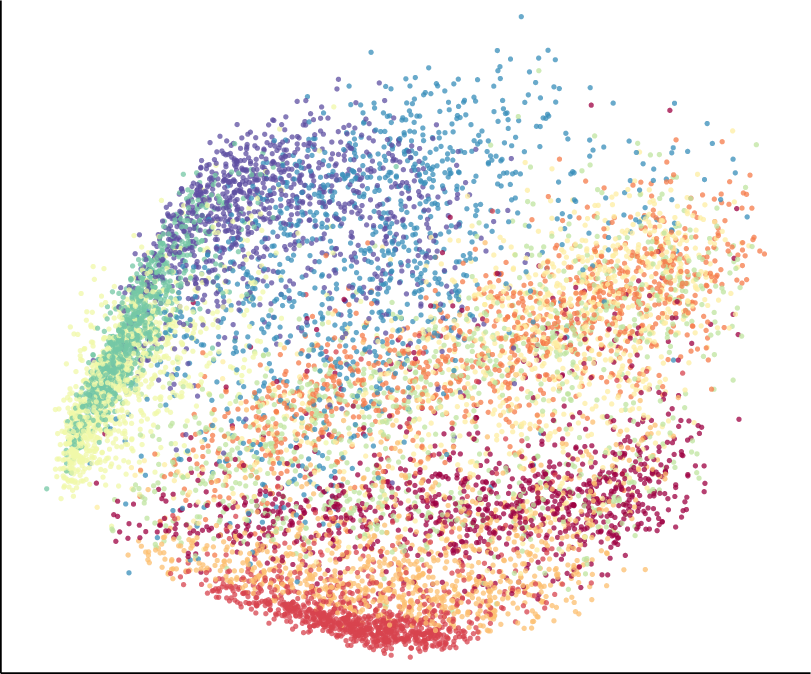}%
  }%
  \subcaptionbox*{PCA}{%
    \includegraphics[width=0.33\linewidth]{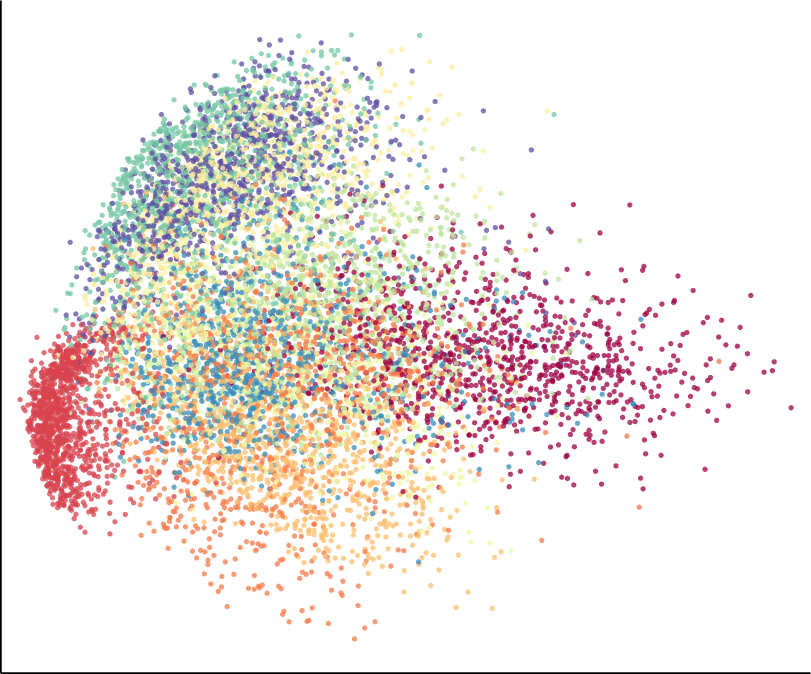}%
  }%
  \subcaptionbox*{}{%
    \includegraphics[width=0.33\linewidth]{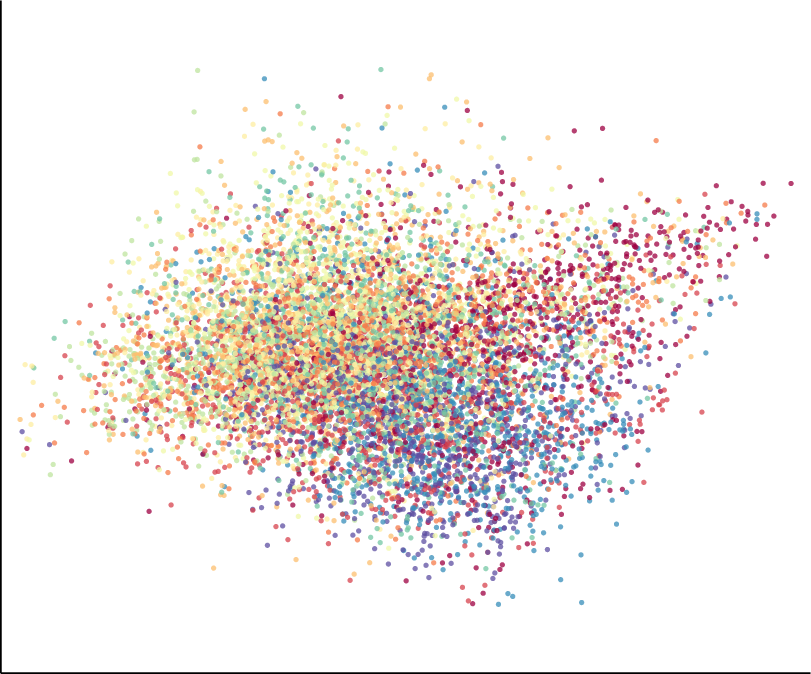}%
  }\\
  \subcaptionbox*{}{%
    \includegraphics[width=0.33\linewidth]{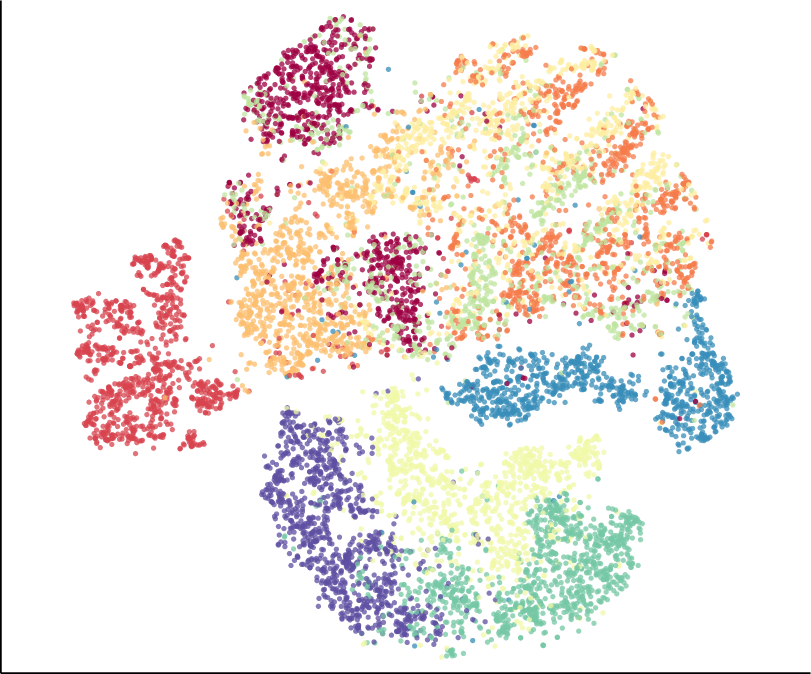}%
  }%
  \subcaptionbox*{t-SNE}{%
    \includegraphics[width=0.33\linewidth]{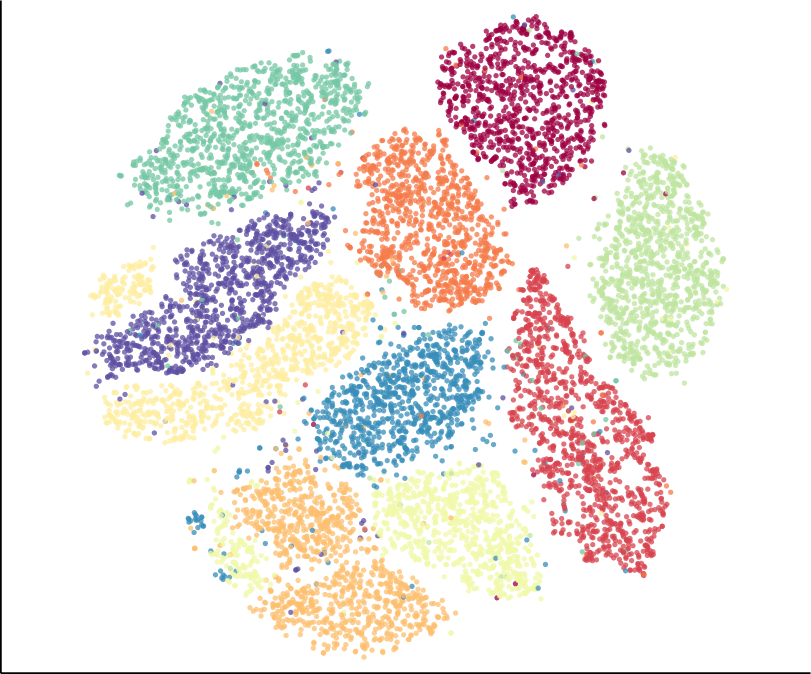}%
  }%
  \subcaptionbox*{}{%
      \includegraphics[width=0.33\linewidth]{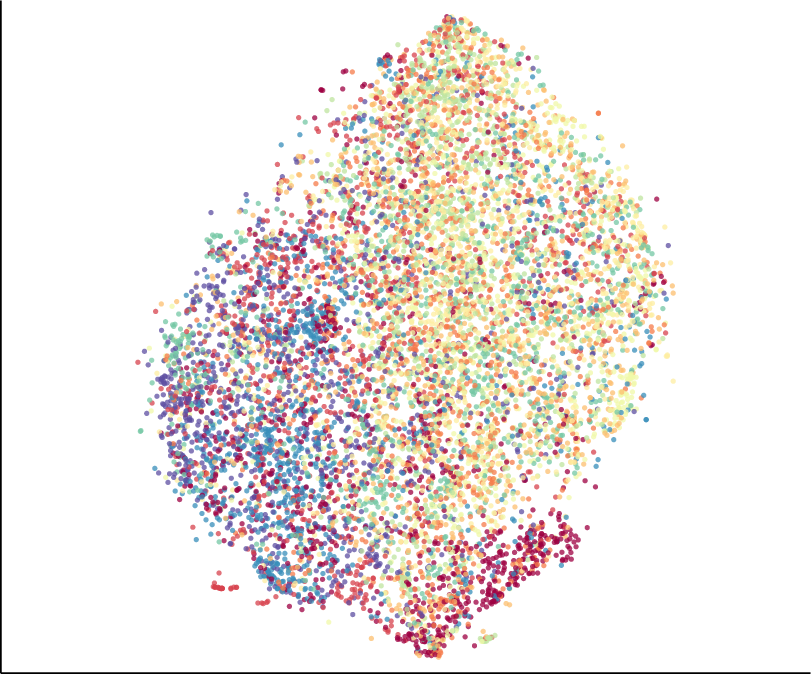}%
  }\\
  \subcaptionbox*{}{%
    \includegraphics[width=0.33\linewidth]{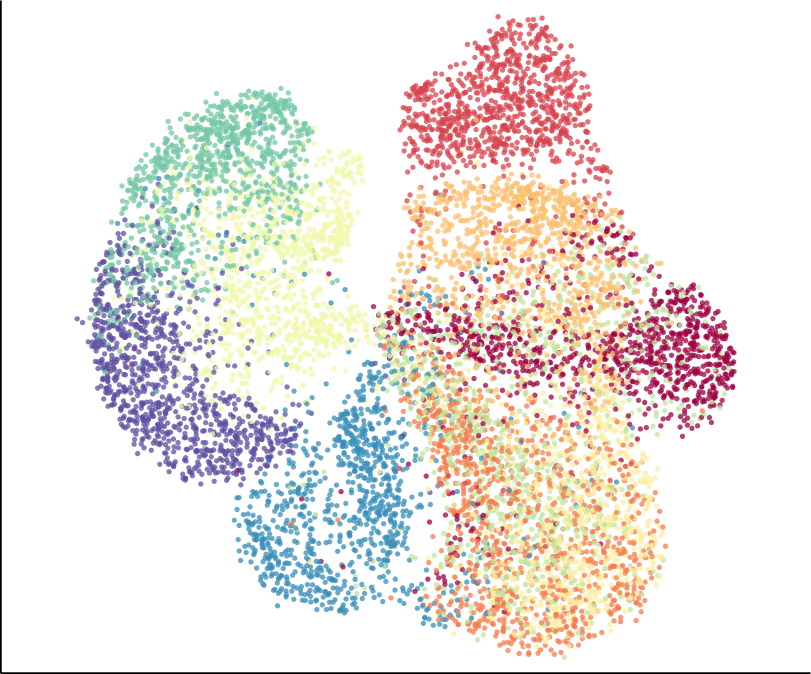}%
  }%
  \subcaptionbox*{UMAP}{%
    \includegraphics[width=0.33\linewidth]{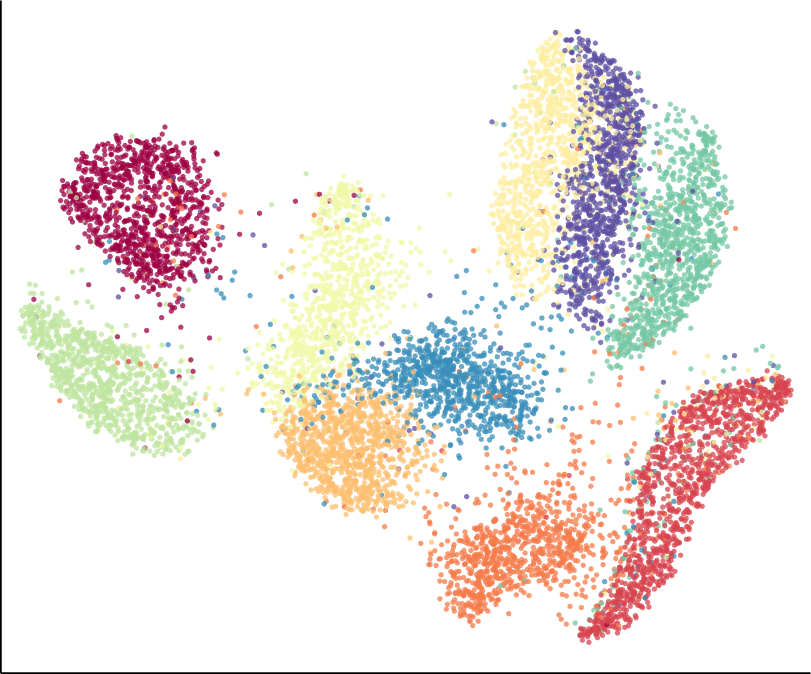}%
  }%
  \subcaptionbox*{}{%
    \includegraphics[width=0.33\linewidth]{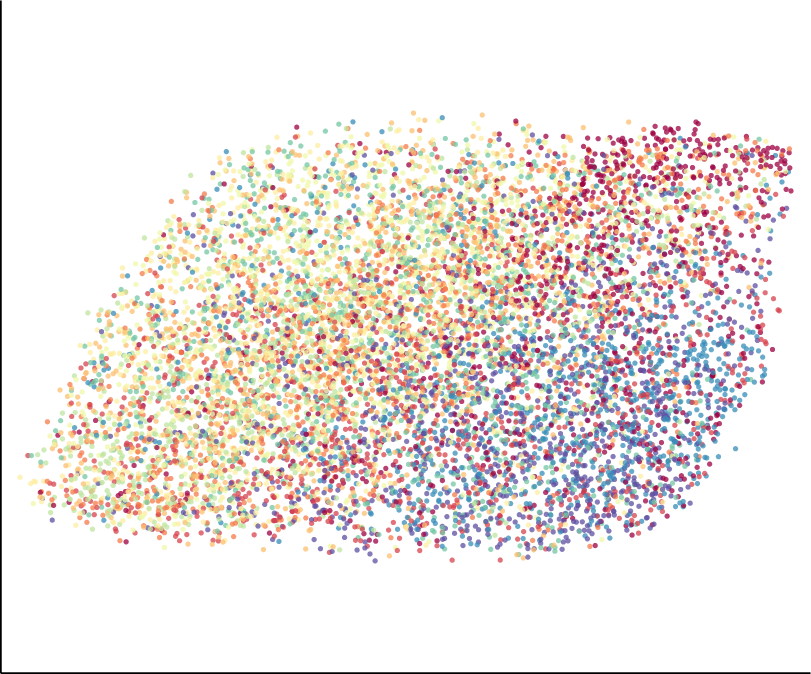}%
  }\\
  \subcaptionbox*{}{%
    \includegraphics[width=0.33\linewidth]{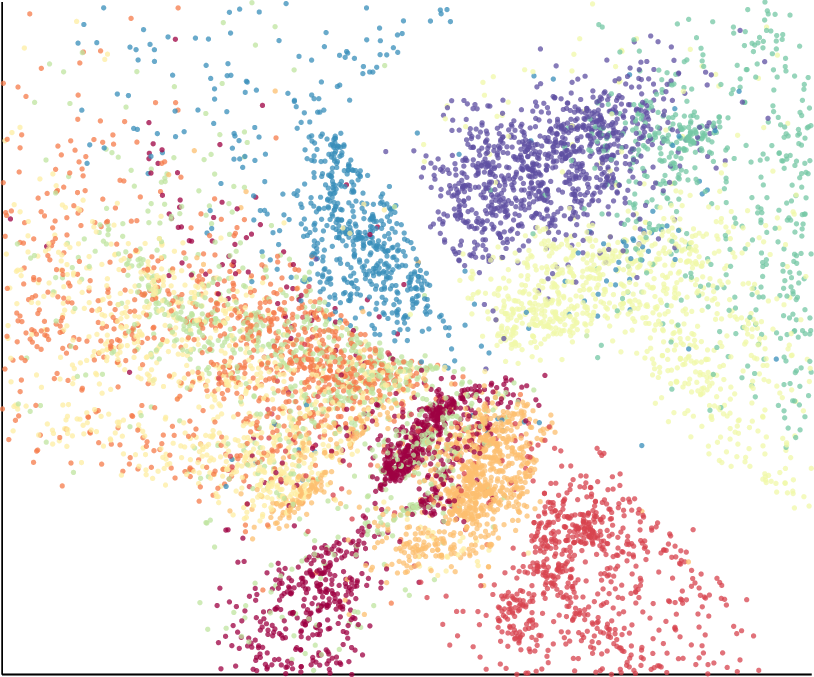}%
  }%
  \subcaptionbox*{AE}{%
    \includegraphics[width=0.33\linewidth]{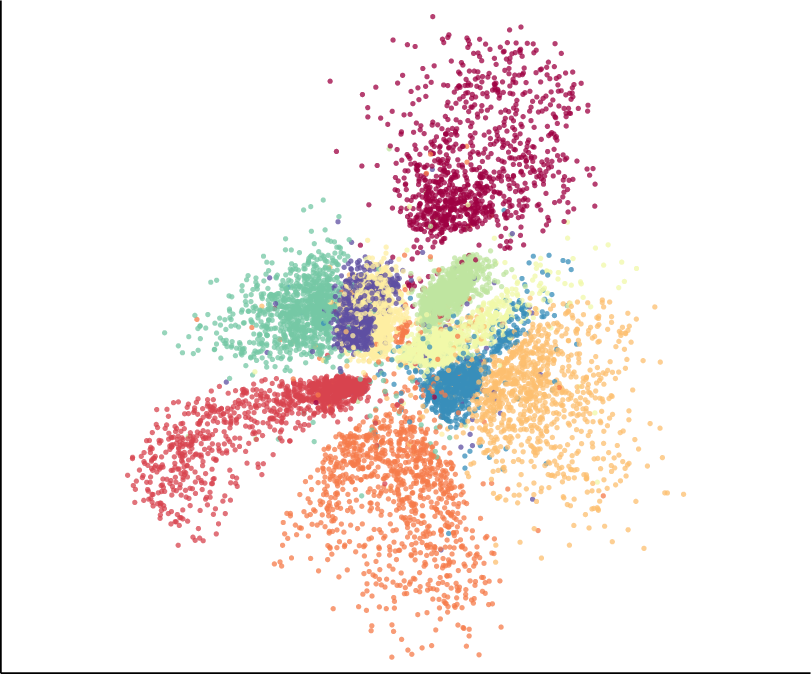}%
  }%
  \subcaptionbox*{}{%
    \includegraphics[width=0.33\linewidth]{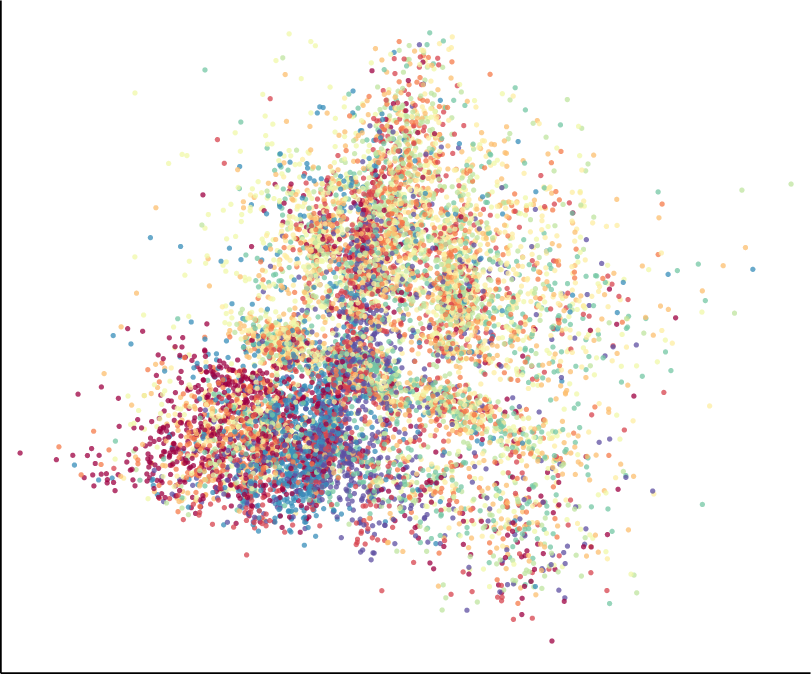}%
  }\\
  \subcaptionbox*{}{%
    \includegraphics[width=0.33\linewidth]{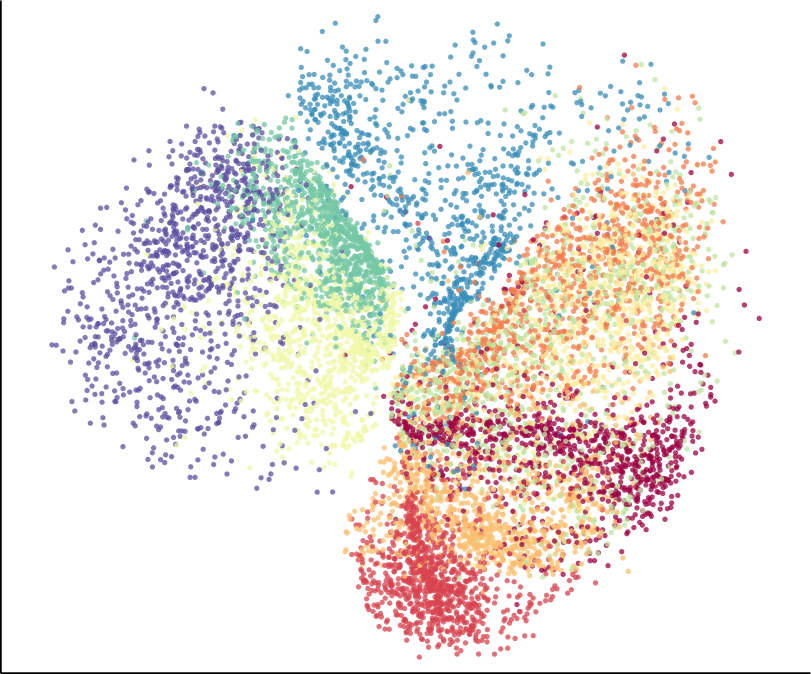}
  }%
  \subcaptionbox*{TopoAE}{%
    \includegraphics[width=0.33\linewidth]{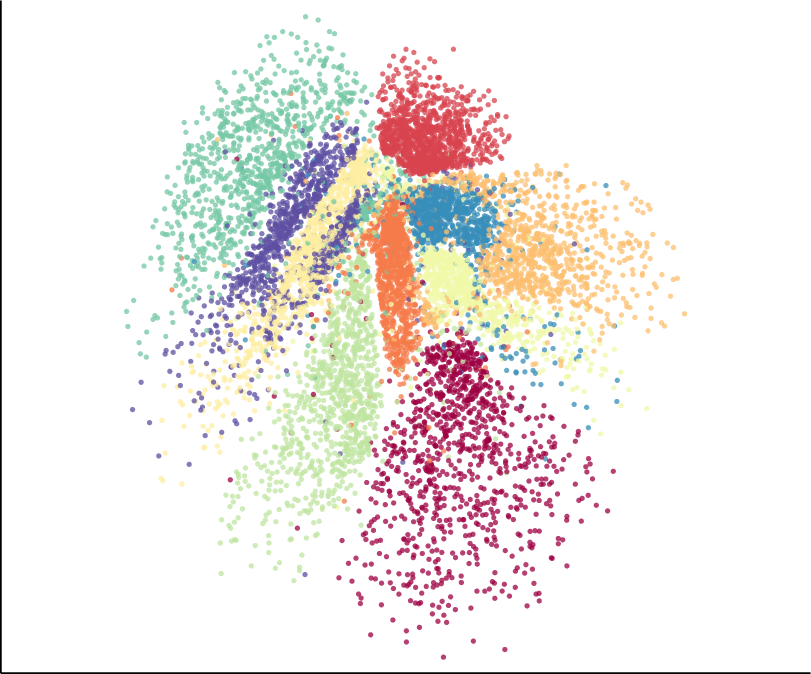}%
  }%
  \subcaptionbox*{}{%
    \includegraphics[width=0.33\linewidth]{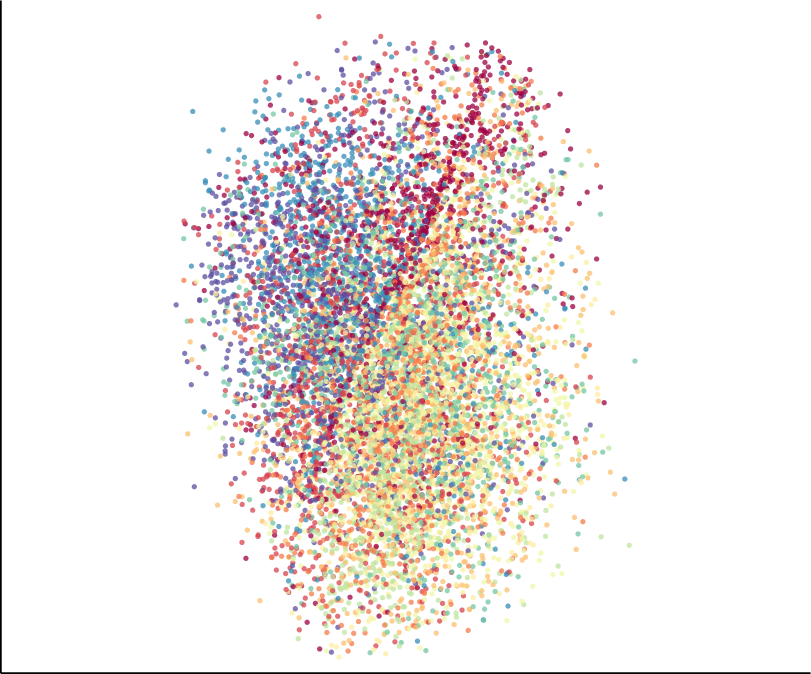}%
  }%
  \caption{%
    Latent representations of the \data{Fashion-MNIST}~(left column),
    \data{MNIST}~(middle column), \data{CIFAR-10}~(right column) data sets.
    The respective batch size values for our method are~$(95, 126, 82)$.
    Please refer to Figures~\ref{fig:FMNIST all methods}, \ref{fig:MNIST
    all methods}, and \ref{fig:CIFAR all methods} in the supplementary
    materials for enlarged versions.
  }
  \label{fig:Summary results}
\end{figure}

\subsection{Results}\label{sec:Results}

Next to a \emph{quantitative} evaluation in terms of various quality
metrics, we also discuss \emph{qualitative} results in terms of
visualisations, which are interpretable in case the ground truth
manifold is known.

\begin{table*}[tbp]
    \centering
{\fontsize{10}{11}\selectfont 
\setlength{\tabcolsep}{2pt}
\begin{tabular}{llrrrrrrrr}
\toprule
 Data set &   Method & $\dkl_{0.01}  $ &  $\dkl_{0.1} $ &   $\dkl_{1} $ &  $\ell$-MRRE &  $\ell$-Cont  & $\ell$-Trust  & $\ell$-RMSE & Data MSE \\
\midrule
\multirow{6}*{\data{Spheres}}     &	Isomap            &                 0.181  &       \second{0.420}  &      \second{0.00881}  &          \second{0.246}  &    \second{0.790}  &   \second{0.676}  &  10.4  &                       --  \\
 & PCA                &        0.332  &       0.651  &      0.01530  &         0.294  &  0.747  &      0.626  &  11.8  &       0.9610  \\
 & TSNE              &        \second{0.152}  &       0.527  &      0.01271  &        \first{0.217}  &   0.773  &      \first{0.679}  &   \first{8.1}  &                       -- \\
 & UMAP              &        0.157 &       0.613  &      0.01658   &         0.250  &   0.752   &      0.635  &   \second{9.3}  &                       --  \\
 & AE           &            0.566  &       0.746  &      0.01664  &        0.349  &     0.607   &    0.588  &  13.3  &       \first{0.8155} \\
 & TopoAE &              \first{0.085}  &       \first{0.326}  &      \first{0.00694}  &        0.272  &     \first{0.822}  &    0.658  &  13.5  &  \second{0.8681}  \\
\midrule 
\multirow{5}*{\data{F-MNIST}} & PCA               &              \first{0.356}  &       \first{0.052}  &      \first{0.00069}  &        0.057  &   0.968  &      0.917  &   \first{9.1}  &       0.1844  \\
& TSNE              &              0.405  &       0.071  &      0.00198  &          \first{0.020}  & 0.967  &      \second{0.974}  &  41.3  &                       -- \\
& UMAP              &                0.424  &       0.065  &      0.00163  &        0.029  &   \first{0.981}  &     0.959   &  \second{13.7}  &                       -- \\
& AE           &       0.478   &       0.068  &      0.00125  &        \second{0.026}  &  0.968  &       \first{0.974}  &  20.7  &       \first{0.1020}  \\
& TopoAE &            \second{0.392}  &       \second{0.054}  &      \second{0.00100}  &         0.032   &    \second{0.980}  &    0.956  &  20.5   &    \second{0.1207}  \\
 \midrule
 \multirow{5}*{\data{MNIST}} & PCA               &        0.389  &       0.163  &  0.00160  &            0.166   & 0.901  &      0.745  &  \first{13.2}  &       0.2227  \\
& TSNE            &        \first{0.277}  &       \second{0.133}  &      0.00214  &         \first{0.040}  &   0.921  &     \first{0.946}  &  22.9  &                       -- \\
&  UMAP              &             \second{0.321}  &       0.146   &      0.00234  &         \second{0.051}  &   \first{0.940}  &     \second{0.938}  &  \second{14.6}  &                       -- \\
&  AE           &               0.620  &       0.155  &      \second{0.00156}  &       0.058  &    0.913   &     0.937  &  18.2   &        \first{0.1373}  \\
& TopoAE  &              0.341  &       \first{0.110}  &      \first{0.00114}  &         0.056  &    \second{0.932}  &    0.928  &  19.6   &       \second{0.1388}  \\
 \midrule
\multirow{5}*{\data{CIFAR}} & PCA               &            \second{0.591}  &       \second{0.020}  &      \first{0.00023}  &        0.119  &  \first{0.931}  &      0.821  &  \first{17.7}  &       0.1482  \\
& TSNE              &                 0.627 &       0.030  &      0.00073  &         \first{0.103}  & 0.903  &       \second{0.863}  &  \second{25.6}  &                       -- \\
&  UMAP              &               0.617  &       0.026  &      0.00050  &         0.127  &    0.920  &    0.817  &  33.6   &                       -- \\
&  AE           &                0.668  &       0.035   &      0.00062  &         0.132  &  0.851   &      \first{0.864}  &  36.3   &       \second{0.1403}  \\
&  TopoAE  &                 \first{0.556}  &       \first{0.019}  &      \second{0.00031}  &          \second{0.108}  &  \second{0.927}  &      0.845  &  37.9  &        \first{0.1398}  \\
\bottomrule
\end{tabular}
}

    \caption{%
      Embedding quality according to multiple evaluation metrics~(Section~\ref{sec:Setup}).
      The hyperparameters of all tunable methods were selected to
      minimise the objective $\dkl_{0.1}$. For each criterion, the
      winner is shown in bold and underlined, the runner-up in bold.
      Please refer to Supplementary Table~\ref{tab:supp_results}
      for more $\sigma$ scales and variance estimates.
      The column ``Data MSE'' indicates the reconstruction error. It is
      included to demonstrate that applying our loss term has no adverse
      effects.
    }
    \label{tab: table_results}
\end{table*}

\subsubsection{Quantitative Results}
%
Table~\ref{tab: table_results} reports the quantitative results.
Overall, we observe that our method is capable of preserving the data
density over multiple length scales~(as measured by $\dkl$).
Furthermore, we find that TopoAE displays competitive continuity
values~($\ell$-Cont) and reconstruction errors~(Data MSE). The latter is
particularly relevant as it demonstrates that imposing our
topological constraints does not result in large impairments when
reconstructing the input space.

The remaining classical measures favour the baselines~(foremost the
\emph{train}~(!) performance of \mbox{t-SNE}).
However, we will subsequently see by inspecting the latent spaces that
those classic measures \emph{fail} to detect the relevant structural
information, as exemplified with known ground truth manifolds, such as
the \data{Spheres} data set.

\subsubsection{Visualisation of Latent Spaces}

For the \data{Spheres} data set~(Figure~\ref{fig:Spheres selected
methods}), we observe that only our method is capable of assessing the
nesting relationship of the high-dimensional spheres correctly.
By contrast, \mbox{t-SNE} ``cuts open'' the enclosing sphere,
distributing most of its points around the remaining spheres. We see
that the $\dkl$-divergence confirms the visual assessment that only our
proposed method preserves the relevant structure of this data set.
Several classical evaluation measures, however, favour \mbox{t-SNE},
even though this method fails to capture the global structure and
nesting relationship of the enclosing sphere manifold accounting for
half of the data set.

On \data{Fashion-MNIST}~(Figure~\ref{fig:Summary results}, leftmost
column), we see that, as opposed to AE, which is purely driven by the reconstruction
error, our method has the additional objective of \emph{preserving}
structure. Here, the constraint helps the regularised autoencoder to ``organise'' the
latent space, resulting in a comparable pattern as in UMAP, which is also
topologically motivated~\citep{mcinnes2018umap}. 
Furthermore, we observe that \mbox{t-SNE} tends to fragment certain
classes~(dark orange, red) into multiple distinct subgroups. This
likely does not reflect the underlying manifold structure, but
constitutes an artefact frequently observed with this method.
For \data{MNIST}, the latent embeddings~(Figure~\ref{fig:Summary
results}, middle column) demonstrate that the non-linear
competitors---mostly by pulling apart distinct classes---lose some of
the relationship information \emph{between} clusters when comparing
against our proposed method or PCA.
Finally, we observe that \data{CIFAR-10}~(Figure~\ref{fig:Summary
results}, rightmost column), is challenging to embed in two latent
dimensions in a purely unsupervised manner.
Interestingly, our method~(consistently, i.e.\ over all runs) was able to
identify a \emph{linear substructure} that separates the latent space in
two additional groups of classes.

\section{Discussion and Conclusion}

We presented topological autoencoders, a novel method for preserving
topological information, measured in terms of persistent homology, of
the input space when learning latent representations with deep neural
networks.
Under weak theoretical assumptions, we showed how our persistent
homology calculations can be combined with backpropagation;
moreover, we proved that approximating persistent homology on the level of mini-batches
is theoretically justified.

In our experiments, we observed that our
method is uniquely able to capture spatial relationships between
nested high-dimensional spheres.
This is relevant, as the ability to cope with \emph{several} manifolds
in the domain of manifold learning still remains a challenging task. On
real-world data sets, we observed that our topological loss leads to
competitive performance in terms of numerous quality metrics~(such as
a density preservation metric), while
\emph{not} adversely affecting the reconstruction error.
In both synthetic and real-world data sets, we obtain interesting and
interpretable representations, as our method does not merely pull apart different
classes, but tries to spatially arrange them meaningfully. Thus, we do
not observe mere distinct ``clouds'', but rather entangled structures,
which we consider to constitute a more meaningful representation of the
underlying manifolds~(an auxiliary analysis in
Supplementary Section~\ref{sec:Topological distances} confirms that our
method influences topological features, measured using PH, in
a beneficial manner).

\paragraph{Future work}
Our topological loss formulation is highly generalisable; it only
requires the existence of a distance matrix between individual
samples~(either globally, or on the level of batches).
As a consequence, our topological loss term can be directly integrated
into a variety of different architectures and is \emph{not} limited to
standard autoencoders. For instance, we can also apply our constraint to
variational setups~(see Figure~\ref{fig:VAE} for a sketch) or create
a topology-aware variant of principal component analysis~(dubbed
``TopoPCA''; see Table~\ref{tab:supp_results} for more details, as well
as Figures~\ref{fig:FMNIST all methods}, \ref{fig:MNIST all methods}, and \ref{fig:CIFAR all methods} for the corresponding
latent space representations).
Employing our loss term to more involved architectures will be
an exciting route for future work.  One issue with the calculation is
that, given the computational complexity of calculating
$\vr_{\epsilon}(\cdot)$, for higher-dimensional features, we would scale
progressively worse with increasing batch size. However, in our
low-dimensional setup, we observed that runtime tends to grow with
decreasing batch size, i.e.\ the mini-batch speed-up still dominates
runtime~(for more details concerning the effect of batch sizes, see
Supplementary Section~\ref{Supp:batch_size}).
In future work, scaling to higher dimensions could be mitigated by
approximating the calculation of persistent homology~\citep{Choudhary18,
Kerber13, Sheehy13} or by exploiting recent advances in parallelising
it~\citep{Bauer14, Lewis15}.
Another interesting extension would be to tackle classification
scenarios with topology-preserving loss terms. This might prove
challenging, however, because the goal in classification is to
increase class separability, which might be achieved by
\emph{removing} topological structures. This goal is therefore at odds
with our loss term that tries \emph{preserving} those structures.
We think that such an extension might require restricting the method to
a subset of scales~(namely those that do not impede class separability)
to be preserved in the data.

\subsubsection*{Acknowledgements}

The authors wish to thank Christian Bock for fruitful discussions and valuable feedback.

This project was supported by the grant \#2017‐110 of the Strategic
Focal Area ``Personalized Health and Related Technologies (PHRT)'' of
the ETH Domain for the SPHN/PHRT Driver Project ``Personalized Swiss
Sepsis Study'' and  the SNSF Starting Grant ``Significant Pattern
Mining''~(K.B., grant no.~155913).
Moreover, this work was funded in part by the Alfried Krupp Prize for Young
University Teachers of the Alfried Krupp von Bohlen und
Halbach-Stiftung~(K.B.).


\bibliography{main}

\begin{thebibliography}{53}
\providecommand{\natexlab}[1]{#1}
\providecommand{\url}[1]{\texttt{#1}}
\expandafter\ifx\csname urlstyle\endcsname\relax
  \providecommand{\doi}[1]{doi: #1}\else
  \providecommand{\doi}{doi: \begingroup \urlstyle{rm}\Url}\fi

\bibitem[Adams et~al.(2017)Adams, Emerson, Kirby, Neville, Peterson, Shipman,
  Chepushtanova, Hanson, Motta, and Ziegelmeier]{Adams17}
Adams, H., Emerson, T., Kirby, M., Neville, R., Peterson, C., Shipman, P.,
  Chepushtanova, S., Hanson, E., Motta, F., and Ziegelmeier, L.
\newblock Persistence images: {A} stable vector representation of persistent
  homology.
\newblock \emph{Journal of Machine Learning Research}, 18\penalty0
  (1):\penalty0 218--252, 2017.

\bibitem[Barannikov(1994)]{Barannikov94}
Barannikov, S.~A.
\newblock The framed {M}orse complex and its invariants.
\newblock \emph{Advances in Soviet Mathematics}, 21:\penalty0 93--115, 1994.

\bibitem[Bauer et~al.(2014)Bauer, Kerber, and Reininghaus]{Bauer14}
Bauer, U., Kerber, M., and Reininghaus, J.
\newblock Distributed computation of persistent homology.
\newblock In McGeoch, C.~C. and Meyer, U. (eds.), \emph{Proceedings of the
  Sixteenth Workshop on Algorithm Engineering and Experiments~(ALENEX)}, pp.\
  31--38. Society for Industrial and Applied Mathematics, 2014.

\bibitem[Bibal \& Fr{\'e}nay(2019)Bibal and Fr{\'e}nay]{Bibal19}
Bibal, A. and Fr{\'e}nay, B.
\newblock Measuring quality and interpretability of dimensionality reduction
  visualizations.
\newblock \emph{Safe Machine Learning Workshop at {ICLR}}, 2019.

\bibitem[Burago et~al.(2001)Burago, Burago, and Ivanov]{Burago01}
Burago, D., Burago, Y., and Ivanov, S.
\newblock \emph{A course in metric geometry}, volume~33 of \emph{Graduate
  Studies in Mathematics}.
\newblock American Mathematical Society, 2001.

\bibitem[Carri\`{e}re et~al.(2015)Carri\`{e}re, Oudot, and
  Ovsjanikov]{Carriere15}
Carri\`{e}re, M., Oudot, S.~Y., and Ovsjanikov, M.
\newblock Stable topological signatures for points on {3D} shapes.
\newblock In \emph{Proceedings of the Eurographics Symposium on Geometry
  Processing~(SGP)}, pp.\  1--12, Aire-la-Ville, Switzerland, 2015.
  Eurographics Association.

\bibitem[Carri{\`e}re et~al.(2019)Carri{\`e}re, Chazal, Ike, Lacombe, Royer,
  and Umeda]{Carriere19}
Carri{\`e}re, M., Chazal, F., Ike, Y., Lacombe, T., Royer, M., and Umeda, Y.
\newblock {P}ers{L}ay: {A} neural network layer for persistence diagrams and
  new graph topological signatures.
\newblock \emph{arXiv e-prints}, art. arXiv:1904.09378, 2019.

\bibitem[Chazal et~al.(2009)Chazal, Cohen-Steiner, Guibas, M{\'e}moli, and
  Oudot]{Chazal09}
Chazal, F., Cohen-Steiner, D., Guibas, L.~J., M{\'e}moli, F., and Oudot, S.~Y.
\newblock {G}romov--{H}ausdorff stable signatures for shapes using persistence.
\newblock \emph{Computer Graphics Forum}, 28\penalty0 (5):\penalty0 1393--1403,
  2009.

\bibitem[Chazal et~al.(2011)Chazal, Cohen-Steiner, and M{\'e}rigot]{Chazal11}
Chazal, F., Cohen-Steiner, D., and M{\'e}rigot, Q.
\newblock Geometric inference for probability measures.
\newblock \emph{Foundations of Computational Mathematics}, 11\penalty0
  (6):\penalty0 733--751, 2011.

\bibitem[Chazal et~al.(2014{\natexlab{a}})Chazal, de~Silva, and
  Oudot]{Chazal14}
Chazal, F., de~Silva, V., and Oudot, S.~Y.
\newblock Persistence stability for geometric complexes.
\newblock \emph{Geometri{\ae} Dedicata}, 173\penalty0 (1):\penalty0 193--214,
  2014{\natexlab{a}}.

\bibitem[Chazal et~al.(2014{\natexlab{b}})Chazal, Fasy, Lecci, Michel, Rinaldo,
  and Wasserman]{Chazal14a}
Chazal, F., Fasy, B.~T., Lecci, F., Michel, B., Rinaldo, A., and Wasserman, L.
\newblock Robust topological inference: {D}istance to a measure and kernel
  distance.
\newblock \emph{arXiv e-prints}, art. arXiv:1412.7197, 2014{\natexlab{b}}.

\bibitem[Chazal et~al.(2015{\natexlab{a}})Chazal, Fasy, Lecci, Michel, Rinaldo,
  and Wasserman]{Chazal15a}
Chazal, F., Fasy, B., Lecci, F., Michel, B., Rinaldo, A., and Wasserman, L.
\newblock Subsampling methods for persistent homology.
\newblock In Bach, F. and Blei, D. (eds.), \emph{Proceedings of the 32nd
  International Conference on Machine Learning~(ICML)}, volume~37 of
  \emph{Proceedings of Machine Learning Research}, pp.\  2143--2151. PMLR,
  2015{\natexlab{a}}.

\bibitem[Chazal et~al.(2015{\natexlab{b}})Chazal, Glisse, Labru{{\`e}}re, and
  Michel]{Chazal15b}
Chazal, F., Glisse, M., Labru{{\`e}}re, C., and Michel, B.
\newblock Convergence rates for persistence diagram estimation in topological
  data analysis.
\newblock \emph{Journal of Machine Learning Research}, 16:\penalty0 3603--3635,
  2015{\natexlab{b}}.

\bibitem[Chen et~al.(2019)Chen, Ni, Bai, and Wang]{Chen19}
Chen, C., Ni, X., Bai, Q., and Wang, Y.
\newblock A topological regularizer for classifiers via persistent homology.
\newblock In Chaudhuri, K. and Sugiyama, M. (eds.), \emph{Proceedings of
  Machine Learning Research}, volume~89 of \emph{Proceedings of Machine
  Learning Research}, pp.\  2573--2582. PMLR, 2019.

\bibitem[Choudhary et~al.(2018)Choudhary, Kerber, and Raghvendra]{Choudhary18}
Choudhary, A., Kerber, M., and Raghvendra, S.
\newblock Improved topological approximations by digitization.
\newblock \emph{arXiv e-prints}, art. arXiv:1812.04966, 2018.

\bibitem[Cohen-Steiner et~al.(2007)Cohen-Steiner, Edelsbrunner, and
  Harer]{Cohen-Steiner07}
Cohen-Steiner, D., Edelsbrunner, H., and Harer, J.
\newblock Stability of persistence diagrams.
\newblock \emph{Discrete {\&} Computational Geometry}, 37\penalty0
  (1):\penalty0 103--120, 2007.

\bibitem[Cormen et~al.(2009)Cormen, Leiserson, Rivest, and Stein]{Cormen09}
Cormen, T.~H., Leiserson, C.~E., Rivest, R.~L., and Stein, C.
\newblock \emph{Introduction to algorithms}.
\newblock MIT Press, Cambridge, MA, USA, 3rd edition, 2009.

\bibitem[Edelsbrunner \& Harer(2008)Edelsbrunner and Harer]{Edelsbrunner08a}
Edelsbrunner, H. and Harer, J.
\newblock Persistent homology---a survey.
\newblock In Goodman, J.~E., Pach, J., and Pollack, R. (eds.), \emph{Surveys on
  discrete and computational geometry: {T}wenty years later}, number 453 in
  Contemporary Mathematics, pp.\  257--282. American Mathematical Society,
  Providence, RI, USA, 2008.

\bibitem[Edelsbrunner et~al.(2002)Edelsbrunner, Letscher, and
  Zomorodian]{Edelsbrunner02}
Edelsbrunner, H., Letscher, D., and Zomorodian, A.~J.
\newblock Topological persistence and simplification.
\newblock \emph{Discrete {\&} Computational Geometry}, 28\penalty0
  (4):\penalty0 511--533, 2002.

\bibitem[Gracia et~al.(2014)Gracia, Gonz{\'a}lez, Robles, and
  Menasalvas]{Gracia14}
Gracia, A., Gonz{\'a}lez, S., Robles, V., and Menasalvas, E.
\newblock A methodology to compare dimensionality reduction algorithms in terms
  of loss of quality.
\newblock \emph{Information Sciences}, 270:\penalty0 1--27, 2014.

\bibitem[Guss \& Salakhutdinov(2018)Guss and Salakhutdinov]{Guss18}
Guss, W.~H. and Salakhutdinov, R.
\newblock On characterizing the capacity of neural networks using algebraic
  topology.
\newblock \emph{arXiv e-prints}, art. arXiv:1802.04443, 2018.

\bibitem[Hinton \& Salakhutdinov(2006)Hinton and
  Salakhutdinov]{hinton2006reducing}
Hinton, G.~E. and Salakhutdinov, R.~R.
\newblock Reducing the dimensionality of data with neural networks.
\newblock \emph{Science}, 313\penalty0 (5786):\penalty0 504--507, 2006.

\bibitem[Hofer et~al.(2017)Hofer, Kwitt, Niethammer, and Uhl]{Hofer17}
Hofer, C., Kwitt, R., Niethammer, M., and Uhl, A.
\newblock Deep learning with topological signatures.
\newblock In Guyon, I., Luxburg, U.~V., Bengio, S., Wallach, H., Fergus, R.,
  Vishwanathan, S., and Garnett, R. (eds.), \emph{Advances in Neural
  Information Processing Systems~30}, pp.\  1633--1643. Curran Associates,
  Inc., 2017.

\bibitem[Hofer et~al.(2019{\natexlab{a}})Hofer, Kwitt, Niethammer, and
  Dixit]{Hofer19a}
Hofer, C., Kwitt, R., Niethammer, M., and Dixit, M.
\newblock Connectivity-optimized representation learning via persistent
  homology.
\newblock In Chaudhuri, K. and Salakhutdinov, R. (eds.), \emph{Proceedings of
  the 36th International Conference on Machine Learning}, volume~97 of
  \emph{Proceedings of Machine Learning Research}, pp.\  2751--2760. PMLR,
  2019{\natexlab{a}}.

\bibitem[Hofer et~al.(2019{\natexlab{b}})Hofer, Graf, Rieck, Niethammer, and
  Kwitt]{Hofer19b}
Hofer, C.~D., Graf, F., Rieck, B., Niethammer, M., and Kwitt, R.
\newblock Graph filtration learning.
\newblock \emph{arXiv e-prints}, art. arXiv:1905.10996, 2019{\natexlab{b}}.

\bibitem[Kerber \& Sharathkumar(2013)Kerber and Sharathkumar]{Kerber13}
Kerber, M. and Sharathkumar, R.
\newblock Approximate {{\v{C}}ech} complexes in low and high dimensions.
\newblock \emph{arXiv e-prints}, art. arXiv:1307.3272, 2013.

\bibitem[Khrulkov \& Oseledets(2018)Khrulkov and Oseledets]{Khrulkov18}
Khrulkov, V. and Oseledets, I.
\newblock Geometry score: {A} method for comparing generative adversarial
  networks.
\newblock In Dy, J. and Krause, A. (eds.), \emph{Proceedings of the 35th
  International Conference on Machine Learning}, volume~80 of \emph{Proceedings
  of Machine Learning Research}, pp.\  2621--2629. PMLR, 2018.

\bibitem[Kingma \& Ba(2014)Kingma and Ba]{kingma2014adam}
Kingma, D.~P. and Ba, J.
\newblock Adam: A method for stochastic optimization.
\newblock \emph{arXiv e-prints}, art. arXiv:1412.6980, 2014.

\bibitem[Kurlin(2015)]{Kurlin15}
Kurlin, V.
\newblock A one-dimensional homologically persistent skeleton of an
  unstructured point cloud in any metric space.
\newblock \emph{Computer Graphics Forum}, 34\penalty0 (5):\penalty0 253--262,
  2015.

\bibitem[Lee et~al.(2003)Lee, Pedersen, and Mumford]{Lee03}
Lee, A.~B., Pedersen, K.~S., and Mumford, D.
\newblock The nonlinear statistics of high-contrast patches in natural images.
\newblock \emph{International Journal of Computer Vision}, 54\penalty0
  (1--3):\penalty0 83--103, 2003.

\bibitem[Lee \& Verleysen(2009)Lee and Verleysen]{Lee09}
Lee, J.~A. and Verleysen, M.
\newblock Quality assessment of dimensionality reduction: {R}ank-based
  criteria.
\newblock \emph{Neurocomputing}, 72\penalty0 (7):\penalty0 1431--1443, 2009.

\bibitem[Lewis \& Morozov(2015)Lewis and Morozov]{Lewis15}
Lewis, R. and Morozov, D.
\newblock Parallel computation of persistent homology using the blowup complex.
\newblock In \emph{Proceedings of the 27th ACM Symposium on Parallelism in
  Algorithms and Architectures~(SPAA)}, pp.\  323--331. ACM, 2015.

\bibitem[McInnes et~al.(2018)McInnes, Healy, and Melville]{mcinnes2018umap}
McInnes, L., Healy, J., and Melville, J.
\newblock {UMAP}: Uniform manifold approximation and projection for dimension
  reduction.
\newblock \emph{arXiv e-prints}, art. arXiv:1802.03426, 2018.

\bibitem[M{\'e}moli \& Sapiro(2004)M{\'e}moli and Sapiro]{Memoli04}
M{\'e}moli, F. and Sapiro, G.
\newblock Comparing point clouds.
\newblock In \emph{Proceedings of the Eurographics/ACM SIGGRAPH Symposium on
  Geometry Processing~(SGP)}, pp.\  32--40, New York, NY, USA, 2004.
  Association for Computing Machinery.

\bibitem[Paul \& Chalup(2017)Paul and Chalup]{Paul17}
Paul, R. and Chalup, S.~K.
\newblock A study on validating non-linear dimensionality reduction using
  persistent homology.
\newblock \emph{Pattern Recognition Letters}, 100:\penalty0 160--166, 2017.

\bibitem[Peyr{\'e}(2009)]{Peyre09}
Peyr{\'e}, G.
\newblock Manifold models for signals and images.
\newblock \emph{Computer Vision and Image Understanding}, 113\penalty0
  (2):\penalty0 249--260, 2009.

\bibitem[Poulenard et~al.(2018)Poulenard, Skraba, and Ovsjanikov]{Poulenard18}
Poulenard, A., Skraba, P., and Ovsjanikov, M.
\newblock Topological function optimization for continuous shape matching.
\newblock \emph{Computer Graphics Forum}, 37\penalty0 (5):\penalty0 13--25,
  2018.

\bibitem[Ramamurthy et~al.(2019)Ramamurthy, Varshney, and Mody]{Ramamurthy19}
Ramamurthy, K.~N., Varshney, K., and Mody, K.
\newblock Topological data analysis of decision boundaries with application to
  model selection.
\newblock In Chaudhuri, K. and Salakhutdinov, R. (eds.), \emph{Proceedings of
  the 36th International Conference on Machine Learning}, volume~97 of
  \emph{Proceedings of Machine Learning Research}, pp.\  5351--5360. PMLR,
  2019.

\bibitem[Reininghaus et~al.(2015)Reininghaus, Huber, Bauer, and
  Kwitt]{Reininghaus15}
Reininghaus, J., Huber, S., Bauer, U., and Kwitt, R.
\newblock A stable multi-scale kernel for topological machine learning.
\newblock In \emph{Proceedings of the IEEE Conference on Computer Vision and
  Pattern Recognition}, pp.\  4741--4748, 2015.

\bibitem[Rieck \& Leitte(2015)Rieck and Leitte]{Rieck15b}
Rieck, B. and Leitte, H.
\newblock Persistent homology for the evaluation of dimensionality reduction
  schemes.
\newblock \emph{Computer Graphics Forum}, 34\penalty0 (3):\penalty0 431--440,
  2015.

\bibitem[Rieck \& Leitte(2017)Rieck and Leitte]{Rieck15a}
Rieck, B. and Leitte, H.
\newblock Agreement analysis of quality measures for dimensionality reduction.
\newblock In Carr, H., Garth, C., and Weinkauf, T. (eds.), \emph{Topological
  Methods in Data Analysis and Visualization IV}. Springer, Cham, Switzerland,
  2017.

\bibitem[Rieck et~al.(2019{\natexlab{a}})Rieck, Bock, and Borgwardt]{Rieck19b}
Rieck, B., Bock, C., and Borgwardt, K.
\newblock A persistent {W}eisfeiler--{L}ehman procedure for graph
  classification.
\newblock In Chaudhuri, K. and Salakhutdinov, R. (eds.), \emph{Proceedings of
  the 36th International Conference on Machine Learning}, volume~97 of
  \emph{Proceedings of Machine Learning Research}, pp.\  5448--5458. PMLR,
  2019{\natexlab{a}}.

\bibitem[Rieck et~al.(2019{\natexlab{b}})Rieck, Togninalli, Bock, Moor, Horn,
  Gumbsch, and Borgwardt]{Rieck19a}
Rieck, B., Togninalli, M., Bock, C., Moor, M., Horn, M., Gumbsch, T., and
  Borgwardt, K.
\newblock Neural persistence: a complexity measure for deep neural networks
  using algebraic topology.
\newblock In \emph{International Conference on Learning
  Representations~(ICLR)}, 2019{\natexlab{b}}.

\bibitem[scikit-optimize contributers(2018)]{scikit-optimize}
scikit-optimize contributers, T.
\newblock \texttt{scikit-optimize}/\texttt{scikit-optimize}: v0.5.2, March
  2018.

\bibitem[Sheehy(2013)]{Sheehy13}
Sheehy, D.~R.
\newblock Linear-size approximations to the {V}ietoris--{R}ips filtration.
\newblock \emph{Discrete {\&} Computational Geometry}, 49\penalty0
  (4):\penalty0 778--796, 2013.

\bibitem[Tenenbaum et~al.(2000)Tenenbaum, {De Silva}, and
  Langford]{tenenbaum2000global}
Tenenbaum, J.~B., {De Silva}, V., and Langford, J.~C.
\newblock A global geometric framework for nonlinear dimensionality reduction.
\newblock \emph{Science}, 290\penalty0 (5500):\penalty0 2319--2323, 2000.

\bibitem[van~der Maaten \& Hinton(2008)van~der Maaten and
  Hinton]{maaten2008visualizing}
van~der Maaten, L.~J. and Hinton, G.
\newblock Visualizing data using {t}-{SNE}.
\newblock \emph{Journal of Machine Learning Research}, 9:\penalty0 2579--2605,
  2008.

\bibitem[van~der Maaten et~al.(2009)van~der Maaten, Postma, and van~den
  Herik]{Maaten09}
van~der Maaten, L.~J., Postma, E.~O., and van~den Herik, H.~J.
\newblock Dimensionality reduction: {A} comparative review.
\newblock Technical Report 2009-005, Tilburg University, 2009.

\bibitem[Venna \& Kaski(2006)Venna and Kaski]{Venna06}
Venna, J. and Kaski, S.
\newblock Visualizing gene interaction graphs with local multidimensional
  scaling.
\newblock In \emph{Proceedings of the 14th European Symposium on Artificial
  Neural Networks}, pp.\  557--562. d-side publishing, 2006.

\bibitem[Vietoris(1927)]{Vietoris27}
Vietoris, L.
\newblock \"{U}ber den h{\"o}heren {Z}usammenhang kompakter {R}{\"a}u\-me und
  eine {K}lasse von zusammenhangstreuen {A}bbildungen.
\newblock \emph{Mathematische Annalen}, 97\penalty0 (1):\penalty0 454--472,
  1927.

\bibitem[Wagner \& D{\l}otko(2014)Wagner and D{\l}otko]{Wagner14}
Wagner, H. and D{\l}otko, P.
\newblock Towards topological analysis of high-dimensional feature spaces.
\newblock \emph{Computer Vision and Image Understanding}, 121:\penalty0 21--26,
  2014.

\bibitem[Yan et~al.(2018)Yan, Zhao, Rosen, Scheidegger, and Wang]{Yan18}
Yan, L., Zhao, Y., Rosen, P., Scheidegger, C., and Wang, B.
\newblock Homology-preserving dimensionality reduction via manifold landmarking
  and tearing.
\newblock \emph{arXiv e-prints}, art. arXiv:1806.08460, 2018.

\bibitem[Zomorodian(2010)]{Zomorodian10a}
Zomorodian, A.~J.
\newblock Fast construction of the {V}ietoris--{R}ips complex.
\newblock \emph{Computers {\&} Graphics}, 34\penalty0 (3):\penalty0 263--271,
  2010.

\end{thebibliography}
\bibliographystyle{icml2020}


\clearpage

\counterwithin{figure}{section}
\counterwithin{table} {section}

\appendix
\section{Appendix}\label{sec:Appendix}

\subsection{Persistent Homology Calculation Details}\label{sec:Persistent homology details}

This section provides more details about the persistent homology
calculation; it is more geared towards an expert reader and aims
for a concise description of all required concepts.

\paragraph{Simplicial homology}
To understand persistent homology, we first have to understand
simplicial homology.
Given a simplicial complex $\simplicialcomplex$, i.e.\
a high-dimensional generalisation of a graph, let
$\chaingroup{d}(\simplicialcomplex)$ denote the vector space generated
over $\mathds{Z}_2$ whose elements are the $d$-simplices in
$\simplicialcomplex$\footnote{%
  It is also possible to describe this calculation with coefficients in
  other fields, but the case of $\mathds{Z}_2$ is advantageous because
  it simplifies the implementation of all operations.
}.
For $\sigma = (v_0,\dots,v_d) \in \simplicialcomplex$, let
$\boundary{d}\colon\chaingroup{d}(\simplicialcomplex)\to\chaingroup{d-1}(\simplicialcomplex)$
be the boundary homomorphism defined by
\begin{equation}
  \boundary{d}(\sigma) := \sum_{i=0}^{d}(v_0,\dots, v_{i-1},v_{i+1},\dots, v_d)
\end{equation}
for a single simplex and linearly extended to $\chaingroup{d}(\simplicialcomplex)$.
The $d$\th homology group $\homologygroup{d}(\simplicialcomplex)$ of $\simplicialcomplex$ is defined as
the quotient group $\homologygroup{d}(\simplicialcomplex) := \ker\boundary{d} / \im\boundary{d+1}$.
The rank of the $d$\th homology group is known as the $d$\th Betti number~$\betti{d}$,
i.e.\ $\betti{d}(\simplicialcomplex) := \rank\homologygroup{d}(\simplicialcomplex)$.
The sequence of Betti numbers $\betti{0},\dots,\betti{d}$ of
a $d$-dimensional simplicial complex is commonly used to distinguish
between different manifolds.
For example, a \mbox{$2$-sphere} in $\real^3$ has Betti numbers
$(1,0,1)$, while a \mbox{$2$-torus} in $\real^3$ has Betti numbers $(1,2,1)$.
Betti numbers are of limited use for analysing real-world data sets,
however, because their representation is too coarse and easily affected
by small changes in the underlying simplicial complex. In an idealised,
platonic setting, this does not pose a problem, because one assumes that
the triangulation of a manifold is known a priori; for real-world data
sets, however, we are typically dealing with point clouds and have
\emph{no} knowledge of the underlying manifold, making the calculation
of the ``proper'' simplicial complex nigh impossible. These disadvantages
prompted the development of persistent homology.

\paragraph{Persistent homology}
Let
$\emptyset = \simplicialcomplex_0 \subseteq \simplicialcomplex_1 \subseteq \dots \subseteq \simplicialcomplex_{m-1} \subseteq \simplicialcomplex_m = \simplicialcomplex$
be a nested sequence of simplicial complexes, called
\emph{filtration}. Filtrations can be defined based on different
functions; the Vietoris--Rips filtration that we discuss in the paper,
for example, is defined by a distance function, such as the Euclidean
distance between points of a point cloud.
Notice that we may still calculate the simplicial homology of each
$\simplicialcomplex_i$ in the filtration. The filtration provides more
information, though: the family of boundary
operators~$\boundary{}(\cdot)$, together with the inclusion
homomorphism, induces a homomorphism between corresponding homology
groups of the filtration, i.e.\ $f_d^{i,j} \colon
\homologygroup{d}(\simplicialcomplex_i) \to
\homologygroup{d}(\simplicialcomplex_j)$. This homomorphism yields
a sequence of homology groups
\begin{equation*}
  \begin{split}
    &0 = \homologygroup{d}(\simplicialcomplex_0) \xrightarrow{f_d^{0,1}} \homologygroup{d}(\simplicialcomplex_1) \xrightarrow{f_d^{1,2}} \dots\\
    & \dots \xrightarrow{f_d^{m-2,m-1}} \homologygroup{d}(\simplicialcomplex_{m-1}) \xrightarrow{f_d^{m-1,m}} \homologygroup{d}(\simplicialcomplex_m) = \homologygroup{d}(\simplicialcomplex)
\end{split}
\end{equation*}
for every dimension $d$. Given indices $i \leq j$, the $d$\th
\emph{persistent homology group} is defined as
\begin{equation}
  \persistenthomologygroup{d}{i,j} := \ker\boundary{d}(\simplicialcomplex_i) / \left( \im\boundary{d+1}(\simplicialcomplex_j)\cap\ker\boundary{d}(\simplicialcomplex_i)\right).
\end{equation}
It can be seen as the homology group that contains all homology classes
created in~$\simplicialcomplex_i$ that are still
\emph{present}~(``active'') in~$\simplicialcomplex_j$.
We define the $d$\th persistent Betti number to be the rank of this
group, i.e.\ $\persistentbetti{d}{i,j} := \rank
\persistenthomologygroup{d}{i,j}$, which generalises the previous
definition for simplicial homology.
Persistent homology results in a \emph{sequence} of Betti
numbers---instead of a single number---that permits a fine-grained
description of topological activity. This activity is typically
summarised in a persistence diagram, thus replacing the indices~$i,j$
with real numbers based on the function that was used to calculate the
filtration.

\paragraph{Persistence diagrams}
%
A filtration often has associated values~(or weights)
$w_0 \leq w_1 \leq \dots \leq w_{m-1} \leq w_m$, such as the pairwise
distances in a point cloud.
These values permit the calculation of topological feature descriptors
known as \emph{persistence diagrams}:
for each dimension~$d$ and each pair $i \leq j$, one stores a pair
$(a,b) := (w_i, w_j) \in \real^2$
with multiplicity
\begin{equation}
  \mu_{i,j}^{(d)} := \left( \persistentbetti{d}{i,j-1} - \persistentbetti{d}{i,j} \right) - \left( \persistentbetti{d}{i-1,j-1} - \persistentbetti{d}{i-1,j} \right)
\end{equation}
in a multiset~(typically, $\mu_{i,j}^{(d)} = 0$ for many pairs).
The pair $(a,b)$ represents a topological feature that was created at
a certain threshold~$a$, and destroyed at another threshold~$b$. In
the case of the Vietoris--Rips filtration and connected components, we
have~$a = 0$  because \emph{all} connected components are present at the beginning of the
filtration by definition. Similarly, $b$ will correspond to an edge used
in the minimum spanning tree of the data set.
In general, the resulting set of points is called the $d$\th
\emph{persistence diagram}~$\diagram_d$.
Given a point $(a,b) \in \diagram_d$, the quantity $\persistence(x,y) :=
|a-b|$ is referred to as its \emph{persistence}.

\subsection{Proof of Theorem~\ref{thm:Probability bottleneck distance}}\label{sec:Proof probability}

\addtocounter{theorem}{-2}

\begin{theorem}
  Let $\pointcloud$ be a point cloud of cardinality $n$ and $\pointcloud^{(m)}$ be one subsample of
  $\pointcloud$ of cardinality $m$, i.e.\ $\pointcloud^{(m)} \subseteq \pointcloud$, sampled without
  replacement. We can bound the probability of $\pointcloud^{(m)}$ exceeding a threshold in terms of
  the bottleneck distance as
  \begin{small}
    \begin{equation}
      \prob\mleft(\db\!\mleft(\diagram^{\pointcloud}\!\!, \diagram^{\pointcloud^{(m)}}\mright)\!>\!\epsilon \mright)\leq \prob\mleft( \dhd\!\mleft(\pointcloud, \pointcloud^{(m)}\mright)\!>\!2\epsilon \mright),
      \label{eqa:Probability bottleneck distance}
    \end{equation}
  \end{small}
  where $\dhd$ refers to the Hausdorff distance between the point cloud and its subsample, i.e.\
  \begin{equation}
    \begin{split}
      \dhd\mleft(X, Y\mright) := \max\{&\sup_{x \in X} \inf_{y \in Y} \dist\mleft(x, y\mright),\\
                                             &\sup_{y \in Y} \inf_{x \in X} \dist\mleft(x, y\mright)\}
    \end{split}
    \label{eqa:Persistence stability}
  \end{equation}
  for a baseline distance $\dist\mleft(x, y\mright)$ such as the Euclidean distance.
\end{theorem}

\begin{proof}
  The stability of persistent homology calculations was proved by \citet{Chazal14} for finite metric
  spaces. More precisely, given two metric spaces $X$ and $Y$, we have
  \begin{equation}
    \db\mleft(\diagram^{X}, \diagram^{Y}\mright) \leq  2\dgh\mleft(X, Y\mright),
  \end{equation}
  where $\dgh\mleft(X, Y\mright)$ refers to the Gromov--Hausdorff distance~\citep[p.\ 254]{Burago01}
  of the two spaces. It is defined as the infimum Hausdorff distance over all isometric embeddings
  of $X$ and $Y$. This distance can be employed for shape comparison~\citep{Chazal09, Memoli04}, but
  is hard to compute. In our case, with $X = \pointcloud$ and $Y = \pointcloud^{(m)}$, we consider
  both spaces to have the same metric~(for $Y$, we take the canonical restriction of the metric from
  $X$ to the subspace $Y$). By definition of the Gromov--Hausdorff distance, we thus have
  $\dgh\mleft(X, Y\mright) \leq \dhd\mleft(X, Y\mright)$, so
  Eq.~\ref{eqa:Persistence stability} leads to
  \begin{equation}
    \db\mleft(\diagram^X, \diagram^Y\mright) \leq 2\dhd\mleft(X, Y\mright),
  \end{equation}
  from which the original claim from Eq.~\ref{eqa:Probability bottleneck distance} follows by taking
  probabilities on both sides.
\end{proof}

\subsection{Empirical Convergence Rates of \texorpdfstring{$\dhd\mleft(\pointcloud, \pointcloud^{(m)}\mright)$}{the Hausdorff Distance}}\label{sec:Empirical convergence rates}

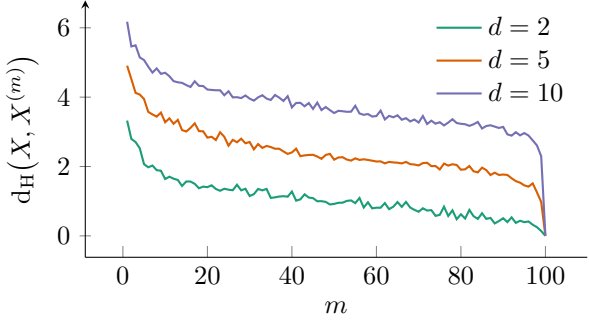
\begin{figure}
  \centering
  \begin{tikzpicture}
    \begin{axis}[%
      axis x line*      = bottom,
      axis y line       = left,
      enlargelimits     = 0.10,
      width             = \linewidth,
      height            = 5cm,
      ylabel            = {$\dhd\mleft(\pointcloud, \pointcloud^{(m)}\mright)$},
      xlabel            = {$m$},
      legend cell align = {left},
      legend style      ={
        draw = none
      },
    ]
      \legend{$d = 2$, $d = 5$, $d = 10$}
      \addplot[thick, Dark2-A] table {data/Hausdorff_subsampling_d02.txt};
      \addplot[thick, Dark2-B] table {data/Hausdorff_subsampling_d05.txt};
      \addplot[thick, Dark2-C] table {data/Hausdorff_subsampling_d10.txt};
    \end{axis}
  \end{tikzpicture}
  \caption{%
    Empirical convergence rate~(mean) of the Hausdorff distance for a subsample of size $m$ of
    $100$ points in a $d$-dimensional space, following a standard normal distribution.
  }
  \label{fig:Hausdorff}
\end{figure}

Figure~\ref{fig:Hausdorff} depicts the mean of the convergence rate~(mean) of the Hausdorff distance
for a subsample of size $m$ of $100$ points in a $d$-dimensional space, following a standard
normal distribution. We can see that the convergence rate is roughly similar, but shown on different
absolute levels that depend on the ambient dimension.
While bounding the convergence rate of
this expression is feasible~\citep{Chazal15a, Chazal15b}, it requires
more involved assumptions on the measures from which $\pointcloud$ and
$\pointcloud^{(m)}$ are sampled.
Additionally, we can give a simple bound using the
\emph{diameter}~$\diameter\mleft(\pointcloud\mright) := \sup\mleft\{\dist\mleft(x, y\mright) \mid x,
y \in \pointcloud \mright\}$. We have $\dhd\mleft(\pointcloud, \pointcloud^{(m)}\mright)
\leq \diameter\mleft(\pointcloud\mright)$ because the supremum is guaranteed to be an upper bound
for the Hausdorff distance. This worst-case bound does not account for
the sample size~(or mini-batch size) $m$, though~(see Theorem~\ref{thm:ExpectedHausdorff} for
an expression that takes $m$ into account).

\subsection{Proof of Theorem~\ref{thm:ExpectedHausdorff}}\label{sec:Proof expectation}

Prior to the proof we state two observations that arise from our special setting of dealing with
finite point clouds.

\begin{observation}
  \label{Ass: hd simple}
  Since $\pointcloud^{(m)} \subseteq \pointcloud$, we have $\sup_{x' \in \pointcloud^{(m)}} \inf_{x
  \in \pointcloud} \dist\mleft(x, x' \mright) = 0$. Hence, the Hausdorff distance simplifies to:
  \begin{equation}
    \label{Eq:hd simplified}
    \dhd\mleft(\pointcloud, \pointcloud^{(m)} \mright) := \sup_{x \in \pointcloud} \inf_{x' \in \pointcloud^{m}} \dist\mleft(x, x' \mright)
  \end{equation}
  In other words, we only have to consider a ``one-sided'' expression of the distance because the
  distance from the subsample to the original point cloud is always zero.
\end{observation}
\begin{observation}
  \label{Ass: finite sets}
  Since our point clouds of interest are finite sets, the suprema and infima of the Hausdorff
  distance coincide with the maxima and minima, which we will subsequently use for easier
  readability.
\end{observation}
Hence, the computation of $\dhd(\pointcloud, \pointcloud^{(m})$ can be divided into three steps.
\begin{enumerate}
  \item Using the baseline distance $\dist\mleft(\cdot, \cdot \mright)$, we compute a distance
    matrix $\mathbf{A} \in \real^{n \times m} $ between all points in $\pointcloud$ and
    $\pointcloud^{(m)}$.
    \item For each of the $n$ points in $\pointcloud$, we compute the minimal distance to the $m$
    samples of $\pointcloud^{(m)}$ by extracting the minimal distance per row of $\mathbf{A}$ and
    gather all minimal distances in $\bm{\delta} \in \real^n$.
    \item Finally, we return the maximal entry of $\bm{\delta}$ as $\dhd\mleft(\pointcloud,
      \pointcloud^{(m)}\mright)$.
\end{enumerate}

In the subsequent proof, we require an independence assumption of the samples.

\begin{proof}
  Using Observations~\ref{Ass: hd simple} and~\ref{Ass: finite sets} we obtain a simplified
  expression for the Hausdorff distance, i.e.\
  \begin{equation}
    \dhd\mleft(\pointcloud, \pointcloud^{(m)}\mright) := \max\limits_{i, 1 \leq i \leq
    n}\mleft(\min\limits_{j,  1 \leq j \leq m}\mleft( a_{ij} \mright)\mright).
  \end{equation}
  The minimal distances of the first $m$ rows of $\mathbf{A}$ are trivially $0$. Hence, the outer
  maximum is determined by the remaining $n-m$ row minima \{$\delta_i \mid m < i \leq n$ \} with
  $\delta_i = \min\limits_{1 \leq j \leq m}a_{ij}$. Those minima follow the
  distribution $F_\Delta(y)$ with
  \begin{align}
    F_\Delta(y) &= \prob(\delta_i \leq y) = 1 - \prob(\delta_i > y) \\
                &= 1 - \prob\mleft( \min_{1 \leq j \leq m} a_{ij} > y \mright)\\
                &= 1 - \prob\mleft( \bigcap_{j} a_{ij} > y \mright) \\ &= 1 - \mleft(1 - F_D\mleft(y\mright)\mright)^{m} \\
                &= - \sum_{k=1}^m {m \choose k}\mleft( - F_D\mleft(y\mright)\mright)^{m-k}.
  \end{align}
  Next, we consider $Z := \max_{1 \leq i \leq n}\delta_{i}$. To evaluate the density of $Z$, we
  first need to derive its distribution $F_Z$:
  \begin{align}
      F_Z(z) &= \prob(Z \leq z) = \prob \mleft( \max_{m < i \leq n} \delta_i \leq z \mright) \\ &= \prob\mleft( \bigcap_{m < i \leq n} \delta_{i} \leq z \mright)
  \end{align}
  Next, we approximate $Z$ by $Z'$ by imposing \emph{i.i.d\ sampling} of the minimal distances
  $\delta_i$ from $F_\Delta$. This is an approximation because in practice, the rows $m+1$ to $n$
  are not stochastically independent because of the triangular inequality that holds for metrics.
  However, assuming i.i.d., we arrive at
  \begin{equation}
    F_{Z'}(z) = F_{\Delta}(z)^{n-m}.
  \end{equation}
  Since $Z'$ has positive support its expectation can then be evaluated as:
  \begin{align}
      \label{Eq:upperbound1}
    \EE_{Z' \sim F_{Z'}}\mleft[ Z' \mright] &= \int\displaylimits_{0}^{+\infty} \mleft( 1 - F_{Z'}(z) \mright)
      \operatorname{d}\!z \\
      &= \int\displaylimits_{0}^{+\infty} \mleft( 1 - F_{\Delta}(z)^{n-m} \mright)
      \operatorname{d}\!z
\end{align}

  The independence assumption leading to $Z'$ results in \emph{overestimating} the variance of the drawn
  minima $\delta_i$. Thus, the expected maximum of those minima, $\EE\mleft[ Z' \mright]$, is
  overestimating the actual expectation of the maximum $\EE\mleft[ Z \mright]$, which is why
  Eq.~\ref{Eq:upperbound1} constitutes an \emph{upper bound} of $\EE\mleft[
    Z \mright]$, and equivalently, an upper bound of $\EE\mleft[\dhd(\pointcloud,
  \pointcloud^{(m)})\mright]$.
  %
\end{proof}

\subsection{Synthetic Data Set}\label{Supp:Datasets}

\data{Spheres} consists of eleven high-dimensional \mbox{$100$-spheres} living in $101-$dimensional
space. Ten spheres of radius $r=5$ are each shifted in a random direction~(by adding the same
Gaussian noise vector per sphere). To this end, we draw ten $d$-dimensional Gaussian vectors
following $\mathcal{N}\mleft(\mathbf{0},\mathbf{I}\mleft(\nicefrac{10}{\sqrt{d}}\mright)\mright)$ for $d=101$.
Crucially, to add interesting topological information to the data set, the ten spheres are enclosed
by an additional larger sphere of radius $5r$. The spheres were generated using the library
\texttt{scikit-tda}.

\subsection{Architectures and Hyperparameter Tuning}\label{Supp:Arch and Hyper}

\paragraph{Architectures for synthetic data set}
For the synthetically generated data set, we use a simple multilayer perceptron
architecture consisting of two hidden layer with $32$ neurons each both encoder and decoder and a
bottleneck of two neurons such that the sequence of hidden-layer neurons is $32-32-2-32-32$. ReLU
non-linearities and batch normalization were applied between the layers excluding the output layer
and the bottleneck layer. The networks were fit using mean squared error loss.

\paragraph{Architectures for real world data sets}
For the \data{MNIST}, \data{Fashion-MNIST}, and \data{CIFAR-10} data sets, we use an architecture inspired by DeepAE~\citep{hinton2006reducing}. This architecture is composed of 3 layers of hidden neurons
of decreasing size ($1000-500-250$) for the encoder part, a bottleneck of two neurons, and a
sequence of three layers of hidden neurons in decreasing size ($250-500-1000$) for the decoder.
In contrast to the originally proposed architecture, we applied ReLU non-linearities and batch
normalization between the layers as we observed faster and more stable training. For the
non-linearities of the final layer, we applied the $tanh$ non-linearity, such that the image of
the activation matches the range of input images scaled between $-1$ and $1$. Also here, we
applied mean squared error loss.

All neural network architectures were fit using Adam and weight-decay of $10^{-5}$.

\paragraph{Hyperparameter tuning}
For hyperparameter tuning we apply random sampling of hyperparameters using the \texttt{scikit-optimize}
library~\citep{scikit-optimize} with 20 calls per method on all data sets. We select the best model parameters in terms of $\dkl_{0.1}$ on the validation split and evaluate and report it on the test split. To estimate performance means and standard deviations, we repeated the evaluation on an independent test split 5 times by using the best parameters (as identified in the hyperparameter search on the validation split) and refitting the models by resampling the train-validation split.

\paragraph{Neural networks}
For the neural networks, we sample the learning rate log-uniformly in
the range $[10^{-4}, 10^{-2}]$, the batch size uniformly between $[16,
128]$, and for our topological autoencoder method~(TopoAE), we sample
the regularisation strength~$\lambda$ log-uniformly in the range
$[10^{-1}, 3]$. Each model was allowed to train for at most 100 epochs
and we applied early stopping with patience $=10$ based on the
validation loss.

\paragraph{Competitor methods}
For \mbox{t-SNE}, we sample the perplexity uniformly in the range $5 - 50$ and the learning rate log-uniformly
in the range $10 - 1000$. For Isomap and UMAP, the number of neighbors included in the computation was varied between
$15 - 500$. For UMAP, we additionally vary the min\_dist parameter uniformly between $0$ and $1$.

\subsection{Measuring the Quality of Latent Representations}\label{sec:Quality measures}
%
Next to the reconstruction error~(if available; please see the paper for
a discussion on this), we use a variety of NLDR metrics to assess the
quality of our method. Our primary interest concerns the quality of the
latent space because, among others, it can be used to visualise the data
set.
We initially considered classical quality metrics from non-linear dimensionality
reduction~(NLDR) algorithms~(see \citet{Bibal19, Gracia14, Maaten09} for
more in-depth descriptions), namely
\begin{compactenum}[(1)]
  \item the \emph{root mean square error}~($\ell$-RMSE) between the distance
  matrix of the original space and the latent space~(as mentioned in the
  main text, this is not related to the reconstruction error),
  \item the \emph{mean relative rank error}~($\ell$-MRRE), which measures
  the changes in \emph{ranks} of distances in the original space and the
  latent space~\citep{Lee09},
  \item the \emph{trustworthiness}~($\ell$-Trust) measure~\citep{Venna06}, which checks to
  what extent the $k$ nearest neighbours of a point are preserved when
  going from the original space to the latent space, and
\item the \emph{continuity}~($\ell$-Cont) measure~\citep{Venna06},
  which is defined analogously to $\ell$-Trust, but checks to what
  extent neighbours are preserved when going from the \emph{latent}
  space to the original space.
\end{compactenum}
All of these measures are defined based on comparisons of the original
space and the latent space; the reconstructed space is \emph{not} used
here.
As an additional measure, we calculate the \emph{Kullback--Leibler
divergence} between density distributions of the input space and the
latent space.
Specifically, for a point cloud $\pointcloud$ with an associated
distance $\dist$, we first use the \emph{distance to a measure}
density estimator~\citep{Chazal11, Chazal14a}, defined as
$\density^{\!\!\mathcal{\pointcloud}}\mleft(x\mright) := \sum_{y \in \mathcal{\pointcloud}} \exp\mleft(-\sigma^{-1}\dist\mleft(x, y\mright)^2\mright)$,
where $\sigma \in \real_{> 0}$ represents a length scale parameter. For $\dist$, we use the Euclidean distance and normalise it between $0$ and $1$.
Given $\sigma$, we evaluate $\dkl_{\sigma} := \dkl\mleft(
\density^{\!\!\pointcloud} \parallel \density^{\!\!Z} \mright)$, which
measures the similarity between the two density distributions.
Ideally, we want the two distributions to be similar because this
implies that density estimates in a low-dimensional representation are
similar to the ones in the original space.

\subsection{Assessing the Batch Size}\label{Supp:batch_size}

As we used fixed architectures for the hyperparameter search, the
batch size remains the main determinant for the runtime of TopoAE. In
Figure \ref{fig:batch_size}, we display trends (linear fits) on how loss
measures vary with batch size. Addtionally, we draw runtime estimates.
As we applied early stopping, for better comparability, we approximated
the epoch-wise runtime by dividing the execution time of a run by its
number of completed epochs. Interestingly, these plots suggest that the
runtime grows with decreasing batch size~(even though the topological
computation is more costly for larger batch sizes!). In these
experiments, sticking to $0-$dimensional topological features we
conclude that the benefit of using mini-batches for neural network training still dominate the topological computations. The few steep peaks most likely represent outliers (the corresponding runs stopped after few epochs, which is why the effective runtime could be overestimated).

For the loss measures, we see that reconstruction loss tends to
\emph{decrease} with increasing batch size, while our topological loss
tends to \emph{increase} with increasing batch size~(despite
normalization). The second observation might be due to larger batch size
enabling more complex data point arrangements and corresponding
topologies.

\begin{figure*}[tbp]
  \centering
  \subcaptionbox{\data{F-MNIST}}{%
    \includegraphics[width=0.25\linewidth]{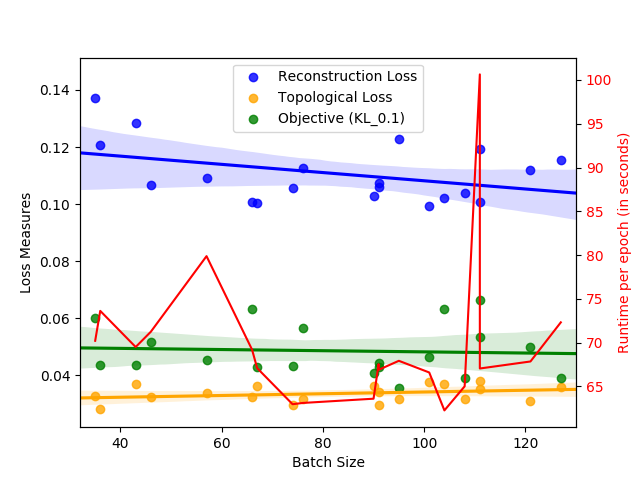}%
  }%
  \subcaptionbox{\data{MNIST}}{%
    \includegraphics[width=0.25\linewidth]{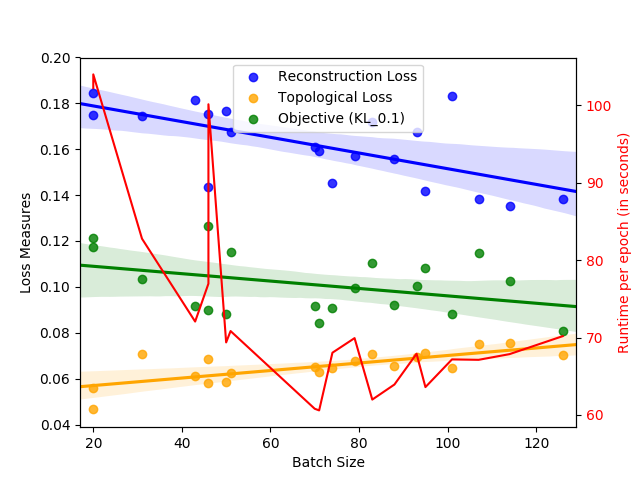}%
  }%
  \subcaptionbox{\data{CIFAR-10}}{%
    \includegraphics[width=0.25\linewidth]{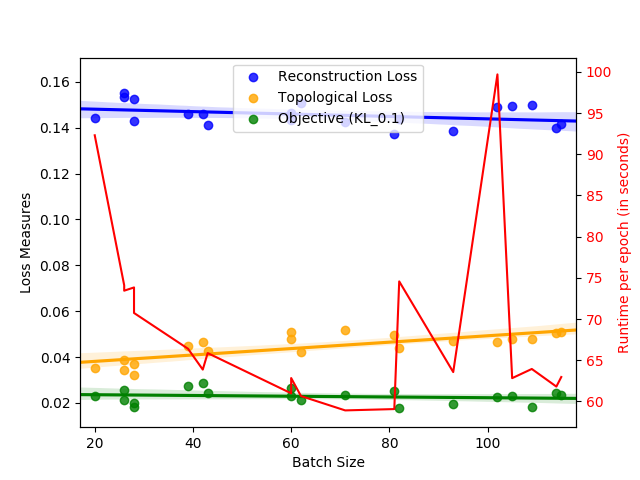}%
  }%
  \subcaptionbox{\data{Spheres}}{%
    \includegraphics[width=0.25\linewidth]{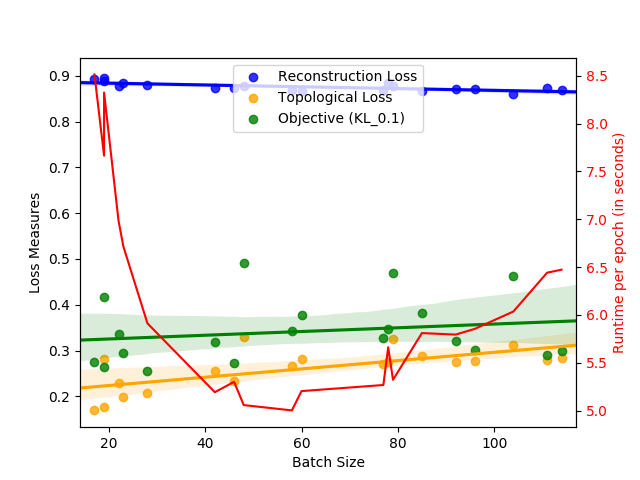}
  }%
  \caption{%
    A scatterplot of batch sizes verses three measures of interest:
    Topological Loss, Reconstruction Loss, and $\dkl_{0.1}$, our
    objective for the hyperparameter search. Additionally, we draw
    per-epoch runtime estimates.
  }
  \label{fig:batch_size}
 \end{figure*}

\subsection{Extending to Variational Autoencoders}\label{Supp:VAE}

In Figure~\ref{fig:VAE} we sketch a preliminary experiment, where we
apply our topological constraint to variational autoencoders for the
\data{Spheres} data set. Also here, we observe that our constraint helps
identifying the nesting structure of the enclosing sphere.

\begin{figure}[tbp]
  \centering
  \subcaptionbox{VAE}{%
    \includegraphics[width=0.50\linewidth]{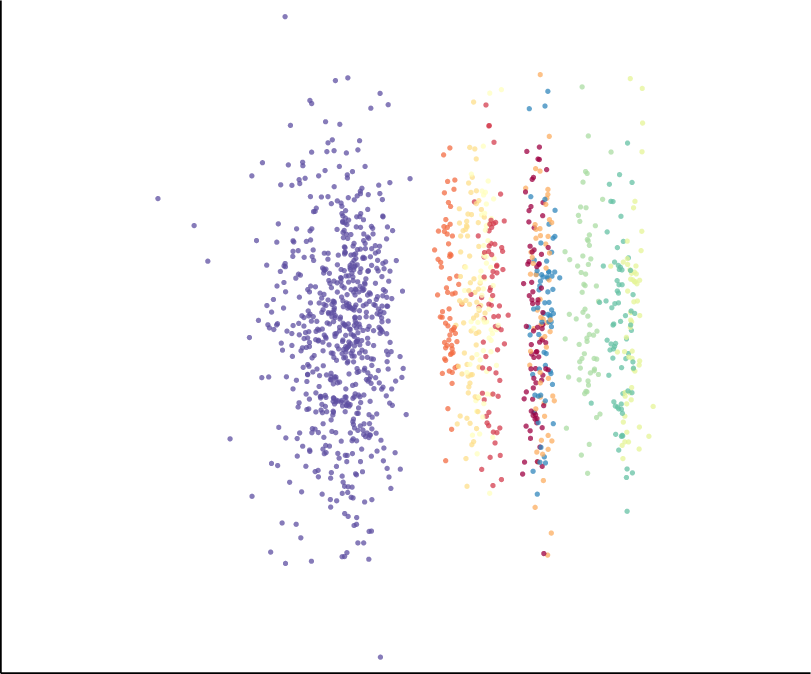}%
  }%
  \subcaptionbox{TopoVAE}{%
    \includegraphics[width=0.50\linewidth]{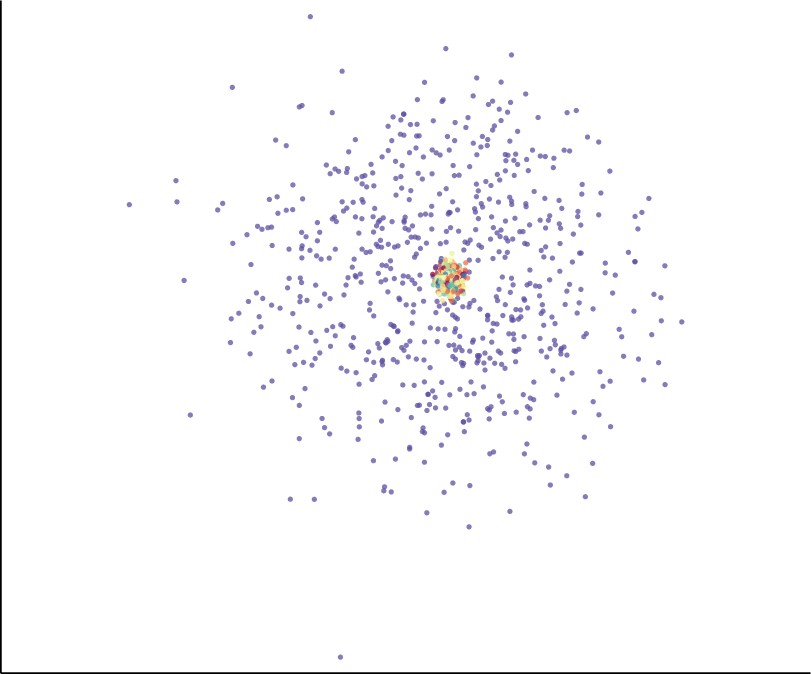}%
  }%
  \caption{%
    A depiction of latent spaces obtained for the \data{Spheres} data
    set with variational autoencoders~(VAEs). Here, VAE represents
    a standard MLP-based VAE, whereas TopoVAE represents the same
    architecture plus our topological constraint.
  }
  \label{fig:VAE}
\end{figure}

\subsection{Topological Distance Calculations\label{sec:Topological distances}}

To assess the topological fidelity of the resulting latent spaces, we
calculate several topological distances between the test data set~(full
dimensionality) and the latent spaces obtained from each method~(two
dimensions). More precisely, we calculate
\begin{inparaenum}[(i)]
  \item the $1$\st Wasserstein distance~($W_1$),
  \item the $2$\nd Wasserstein distance~($W_2$), and
  \item the bottleneck distance~($W_\infty$)
\end{inparaenum}
between the persistence diagrams obtained from the test data set of the
\data{Spheres} data and their resulting 2D latent representations. Even
though our loss function is \emph{not} optimising this distance, we
observe in Table~\ref{tab:Topological distances} that the topological
distance of our method~(``TopoAE'') is always the lowest among all the
methods. In particular, it is \emph{always} smaller than the topological
distance of the latent space of the autoencoder architecture; this is
true for all distance measures, even though $W_\infty$, for example, is
known to be susceptible to outliers. Said experiment serves as a simple
``sanity check'' as it demonstrates that the changes induced by our
method are beneficial in that they reduce the topological distance of
the latent space to the original data set. For a proper comparison of
topological features between the two sets of spaces, a more involved
approach would be required, though.

\begin{table}[tbp]
  \centering
  \small
  \begin{tabular}{lrrr}
    \toprule
    Method & $W_1$ & $W_2$ & $W_\infty$\\
    \midrule
    Isomap & 4.32$\pm$0.037 & 0.477$\pm$0.0045 & 0.165$\pm$0.00096\\
    PCA    & 4.42$\pm$0.053 & 0.476$\pm$0.0046 & 0.158$\pm$0.00108\\
    t-SNE  & 4.38$\pm$0.038 & 0.478$\pm$0.0045 & 0.164$\pm$0.00094\\
    UMAP   & 4.47$\pm$0.042 & 0.478$\pm$0.0045 & 0.160$\pm$0.00092\\
    AE     & 3.99$\pm$0.037 & 0.469$\pm$0.0053 & 0.154$\pm$0.00128\\
    \midrule
    TopoAE & 3.73$\pm$0.076 & 0.459$\pm$0.0055 & 0.152$\pm$0.00268\\
    \bottomrule
  \end{tabular}
  \caption{%
    Topological distances between the test data set and the
    corresponding latent space. We used subsamples of size $m = 500$ and
    $10$ repetitions~(obtaining a mean and a standard deviation).
  }
  \label{tab:Topological distances}
\end{table}

\subsection{Alternative Loss Formulations}\label{sec:Alternative loss}

\begin{figure}[tb]
  \centering
  \subcaptionbox{$\dspace$}{%
    \begin{tikzpicture}[scale=1.25]
      \coordinate (A) at (0.00, 0.00);
      \coordinate (B) at (1.00,-0.25);
      \coordinate (C) at (0.50, 1.25);

      \node[anchor=north] at (A) {A};
      \node[anchor=north] at (B) {B};
      \node[anchor=south] at (C) {C};

      \foreach \x in {A, B, C}
        \filldraw[black] (\x) circle (0.03cm);

      \draw[dotted] (A) -- node[anchor=north] {$d_1$} (B);
      \draw[dotted] (A) -- node[anchor=east ] {$d_3$} (C);
      \draw[dotted] (B) -- node[anchor=west ] {$d_2$} (C);

    \end{tikzpicture}
  }
  \quad
  \subcaptionbox{$\lspace$}{%
    \begin{tikzpicture}[scale=1.25]
      \coordinate (A) at (0.50, 1.25);
      \coordinate (B) at (1.00,-0.25);
      \coordinate (C) at (0.00, 0.00);

      \node[anchor=south] at (A) {A};
      \node[anchor=north] at (B) {B};
      \node[anchor=north] at (C) {C};

      \foreach \x in {A, B, C}
        \filldraw[black] (\x) circle (0.03cm);

     \draw[dotted] (A) -- node[anchor=west ] {$d_2$} (B);
     \draw[dotted] (A) -- node[anchor=east ] {$d_3$} (C);
     \draw[dotted] (B) -- node[anchor=north] {$d_1$} (C);

    \end{tikzpicture}
  }
  \caption{%
    An undesirable configuration of the latent space of three
    non-collinear points, resulting in equal persistence diagrams for
    $\dspace$ and $\lspace$. Pairwise distances are shown as dotted lines.
    We prevent this by not
    \emph{explicitly} minimising the distances between persistence
    diagrams but by including persistence pairings.
  }
  \label{fig:Counterexample}
\end{figure}
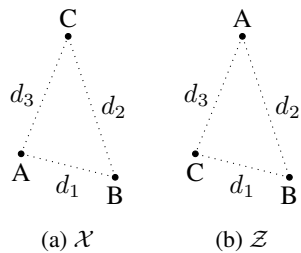

Our choice of loss function was motivated by the observation that
\emph{only} aligning the persistence diagrams between mini-batches of
$\dspace$ and $\lspace$ can lead to degenerate or ``meaningless'' latent
spaces. As a simple example~(see Figure~\ref{fig:Counterexample} for
a visualisation), imagine three non-collinear points in the
input space and the triangle they are forming. Now assume that the
latent space consists of the same triangle~(in terms of its side
lengths) but with permuted labels. A loss term of the form
\begin{equation}
 \loss' := \mleft\| \mat{A}^{\pointcloud}\mleft[ \Index^{\pointcloud} \mright] - \mat{A}^{Z}\mleft[ \Index^{Z} \mright] \mright\|^2
\end{equation}
only measures the distance between persistence diagrams~(which would be
zero in this situation) and would not be able to penalise such a configuration.

\begin{figure*}[t]
  \centering
  \subcaptionbox{PCA}{%
    \includegraphics[width=0.33\linewidth]{figures/spheres/PCA.png}
  }%
  \subcaptionbox{Isomap}{%
    \includegraphics[width=0.33\linewidth]{figures/spheres/Isomap.png}
  }%
  \subcaptionbox{t-SNE}{%
    \includegraphics[width=0.33\linewidth]{figures/spheres/t-SNE.png}
  }\\
  \subcaptionbox{UMAP}{%
    \includegraphics[width=0.33\linewidth]{figures/spheres/UMAP.png}
  }%
  \subcaptionbox{AE}{%
    \includegraphics[width=0.33\linewidth]{figures/spheres/Vanilla.png}
  }
  \subcaptionbox{TopoPCA}{%
    \includegraphics[width=0.33\linewidth]{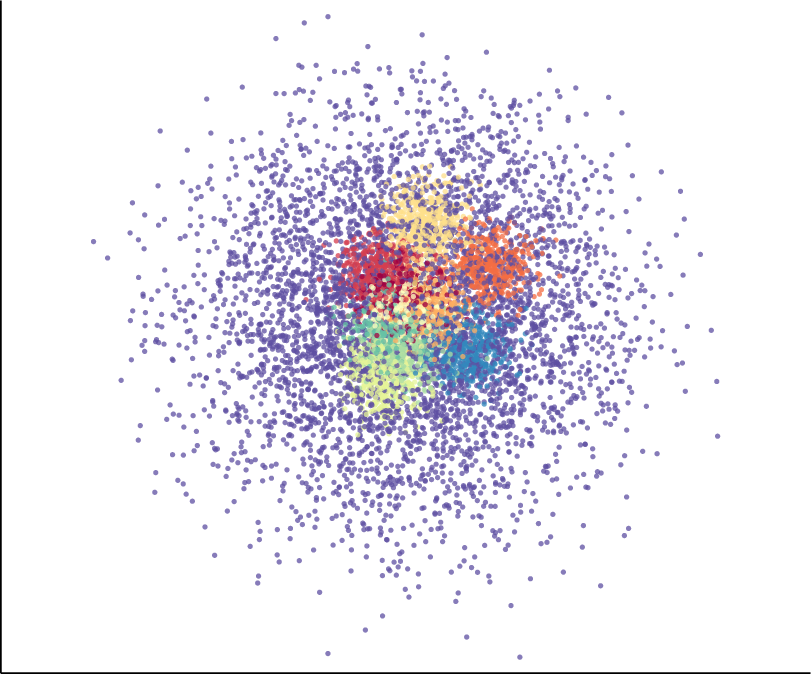}
  }
  \subcaptionbox{TopoAE}{%
    \includegraphics[width=0.33\linewidth]{figures/spheres/TopRegEdgeSymmetric.png}
  }%
  \caption{%
    A depiction of \emph{all} latent spaces obtained for the
    \data{Spheres} data set. TopoAE used a batch size of $28$. This is an enlarged
    version of the figure shown in Section~\ref{sec:Results}.  }
  \label{fig:Spheres all methods}
\end{figure*}

\begin{figure*}[t]
  \centering
  \subcaptionbox{PCA\label{sfig:FMNIST PCA}}{%
    \includegraphics[width=0.33\linewidth]{figures/FashionMNIST/PCA.png}
  }%
  \subcaptionbox{\mbox{t-SNE}}{%
    \includegraphics[width=0.33\linewidth]{figures/FashionMNIST/t-SNE.png}
  }%
  \subcaptionbox{UMAP}{%
    \includegraphics[width=0.33\linewidth]{figures/FashionMNIST/UMAP.png}
  }\\%
  \subcaptionbox{AE}{%
    \includegraphics[width=0.33\linewidth]{figures/FashionMNIST/Vanilla.png}
  }%
  \subcaptionbox{TopoPCA}{%
    \includegraphics[width=0.33\linewidth]{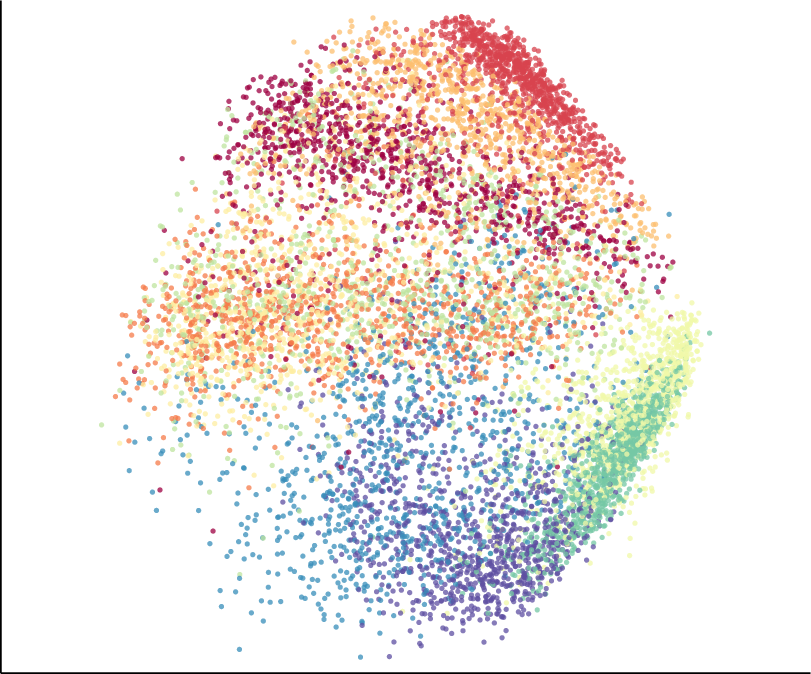}
  }%
  \subcaptionbox{TopoAE}{%
    \includegraphics[width=0.33\linewidth]{figures/FashionMNIST/TopRegEdgeSymmetric.png}
  }%
  \caption{%
    Latent representations of the \data{Fashion-MNIST} data set. TopoAE
    used a batch size of $95$. This is a larger extension
    of the figure shown in Section~\ref{sec:Results}.
  }
  \label{fig:FMNIST all methods}
\end{figure*}

\begin{figure*}[t]
  \centering
  \subcaptionbox{PCA}{%
    \includegraphics[width=0.33\linewidth]{figures/MNIST/PCA.png}
  }%
  \subcaptionbox{\mbox{t-SNE}}{%
    \includegraphics[width=0.33\linewidth]{figures/MNIST/t-SNE.png}
  }%
  \subcaptionbox{UMAP}{%
    \includegraphics[width=0.33\linewidth]{figures/MNIST/UMAP.png}
  }\\%
   \subcaptionbox{AE}{%
    \includegraphics[width=0.33\linewidth]{figures/MNIST/Vanilla.png}
  }%
  \subcaptionbox{TopoPCA}{%
    \includegraphics[width=0.33\linewidth]{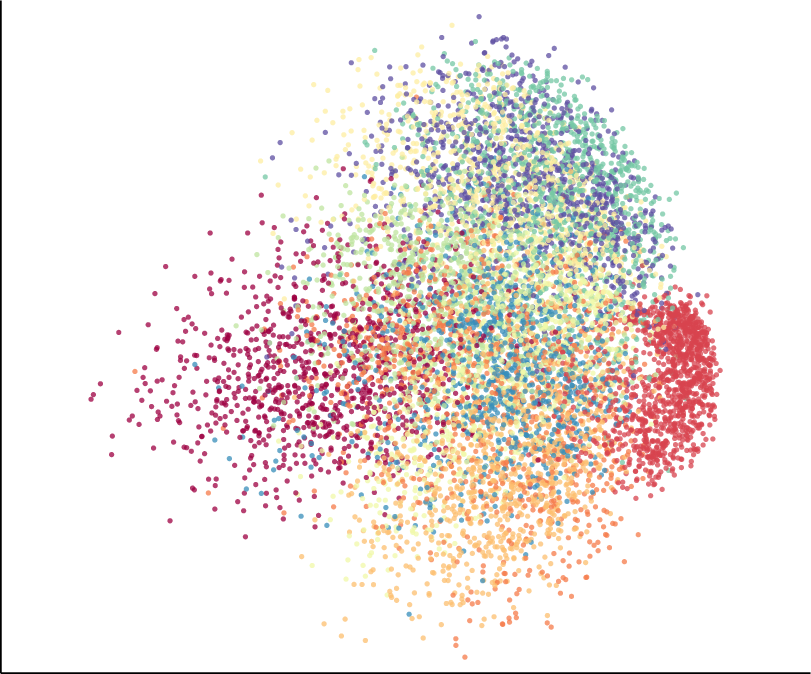}
  }%
  \subcaptionbox{TopoAE}{%
    \includegraphics[width=0.33\linewidth]{figures/MNIST/TopRegEdgeSymmetric.png}
  }%
  \caption{%
    Latent representations of the \data{MNIST} data set. TopoAE used
    a batch size of $126$. This is a larger extension
    of the figure shown in Section~\ref{sec:Results}.
  }
  \label{fig:MNIST all methods}
\end{figure*}

\begin{figure*}[t]
  \centering
  \subcaptionbox{PCA}{%
    \includegraphics[width=0.33\linewidth]{figures/CIFAR/PCA.png}
  }%
  \subcaptionbox{\mbox{t-SNE}}{%
    \includegraphics[width=0.33\linewidth]{figures/CIFAR/t-SNE.png}
  }%
  \subcaptionbox{UMAP}{%
    \includegraphics[width=0.33\linewidth]{figures/CIFAR/UMAP.png}
  }\\%
   \subcaptionbox{AE}{%
    \includegraphics[width=0.33\linewidth]{figures/CIFAR/Vanilla.png}
  }%
  \subcaptionbox{TopoPCA}{%
    \includegraphics[width=0.33\linewidth]{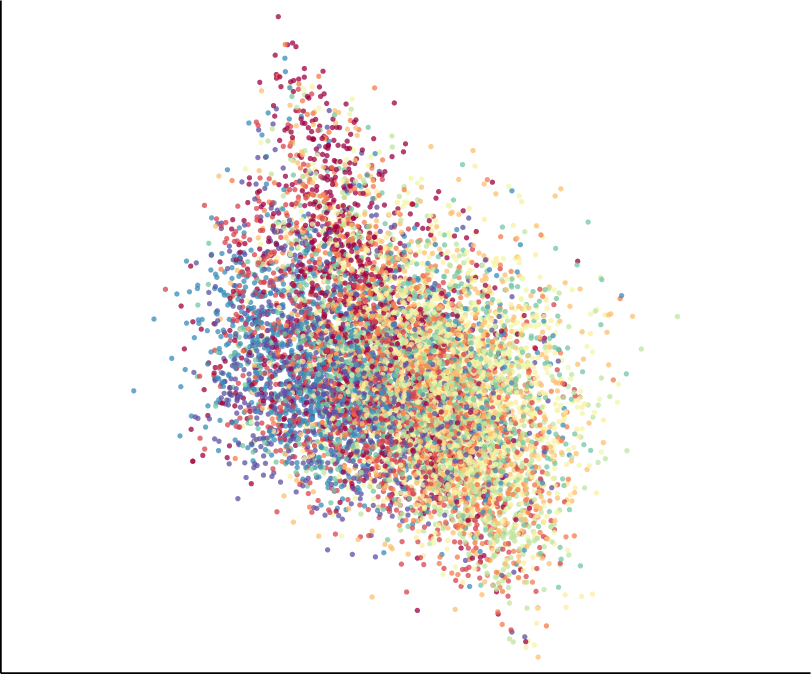}
    }%
  \subcaptionbox{TopoAE}{%
    \includegraphics[width=0.33\linewidth]{figures/CIFAR/TopRegEdgeSymmetric.png}
  }%
  \caption{%
    Latent representations of the \data{CIFAR-10} data set. TopoAE used
    a batch size of $82$. This is a larger extension of the figure shown in Section~\ref{sec:Results}.
  }
  \label{fig:CIFAR all methods}
\end{figure*}

\clearpage

\begin{sidewaystable*}
    \centering
    {\fontsize{7.0}{7.5}\selectfont
\setlength{\tabcolsep}{1pt}
\begin{tabular}{llrrrrrrrrrr}
\toprule
 Data set &   Method & $\dkl_{0.001}$ & $\dkl_{0.01}$ &  $\dkl_{0.1}$ &   $\dkl_{1}$ & $\dkl_{10}$ & $\ell$-Cont & $\ell$-MRRE &  $\ell$-Trust & $\ell$-RMSE & Data MSE \\
\midrule
\multirow{7}*{\data{Spheres}}     &	Isomap            &              0.53095$\pm$0.01929 &              0.18096$\pm$0.02547 &              0.42048$\pm$0.00559 &              0.00881$\pm$0.00020 &              0.000089$\pm$0.000002 &   \second{0.79027$\pm$0.00244} &   \second{0.24573$\pm$0.00158} &   \second{0.67643$\pm$0.00323} &             10.37188$\pm$0.22856 &                                -- \\
& PCA               &     \first{0.22445$\pm$0.00691} &              0.33231$\pm$0.00552 &              0.65121$\pm$0.00256 &              0.01530$\pm$0.00010 &              0.000159$\pm$0.000001 &              0.74740$\pm$0.00140 &              0.29402$\pm$0.00108 &              0.62557$\pm$0.00066 &             11.76482$\pm$0.01460 &              0.96103$\pm$0.00029 \\
& TSNE              &   \second{0.22794$\pm$0.00722} &   \second{0.15228$\pm$0.00805} &              0.52722$\pm$0.03261 &              0.01271$\pm$0.00058 &              0.000133$\pm$0.000006 &              0.77300$\pm$0.00513 &     \first{0.21740$\pm$0.00472} &     \first{0.67862$\pm$0.00474} &     \first{8.05018$\pm$0.11057} &                                -- \\
& UMAP              &              0.24752$\pm$0.01917 &              0.15687$\pm$0.00599 &              0.61326$\pm$0.00752 &              0.01658$\pm$0.00028 &              0.000178$\pm$0.000003 &              0.75153$\pm$0.00360 &              0.24968$\pm$0.00094 &              0.63483$\pm$0.00185 &   \second{9.27009$\pm$0.03417} &                                -- \\
& Vanilla           &              0.28432$\pm$0.02165 &              0.56571$\pm$0.02864 &              0.74588$\pm$0.04323 &              0.01664$\pm$0.00115 &              0.000172$\pm$0.000013 &              0.60663$\pm$0.01685 &              0.34918$\pm$0.00903 &              0.58843$\pm$0.00475 &             13.33061$\pm$0.05198 &     \first{0.81545$\pm$0.00106} \\
& TopoPCA           &              0.43344$\pm$0.01823 &              0.17837$\pm$0.00888 &   \second{0.39816$\pm$0.01178} &   \second{0.00866$\pm$0.00025} &   \second{0.000087$\pm$0.000003} &              0.77320$\pm$0.00135 &              0.32765$\pm$0.00158 &              0.62260$\pm$0.00251 &             11.91542$\pm$0.56134 &              0.97305$\pm$0.00067 \\
& TopoAE &              0.62765$\pm$0.05415 &     \first{0.08504$\pm$0.01270} &     \first{0.32572$\pm$0.02050} &     \first{0.00694$\pm$0.00055} &     \first{0.000069$\pm$0.000006} &     \first{0.82200$\pm$0.01813} &              0.27239$\pm$0.01108 &              0.65775$\pm$0.01428 &             13.45753$\pm$0.04177 &   \second{0.86812$\pm$0.00074} \\
\midrule

\multirow{6}*{\data{F-MNIST}} & PCA               &              0.22559$\pm$0.00011 &     \first{0.35594$\pm$0.00011} &     \first{0.05205$\pm$0.00004} &     \first{0.00069$\pm$0.00000} &     \first{0.000007$\pm$0.000000} &              0.96777$\pm$0.00001 &              0.05744$\pm$0.00001 &              0.91681$\pm$0.00003 &      \first{9.05121$\pm$0.00041} &              0.18439$\pm$0.00000 \\
& TSNE              &     \first{0.03516$\pm$0.00226} &              0.40477$\pm$0.01251 &              0.07095$\pm$0.00962 &              0.00198$\pm$0.00026 &              0.000023$\pm$0.000003 &              0.96731$\pm$0.00268 &     \first{0.01962$\pm$0.00073} &   \second{0.97405$\pm$0.00070} &              41.25460$\pm$0.53671 &                                -- \\
& UMAP              &   \second{0.05069$\pm$0.00238} &              0.42362$\pm$0.00609 &              0.06491$\pm$0.00161 &              0.00163$\pm$0.00005 &              0.000019$\pm$0.000001 &     \first{0.98126$\pm$0.00016} &              0.02867$\pm$0.00034 &              0.95874$\pm$0.00060 &   \second{13.68933$\pm$0.02896} &                                -- \\
& Vanilla           &              0.17177$\pm$0.13603 &              0.47798$\pm$0.09567 &              0.06791$\pm$0.00700 &              0.00125$\pm$0.00017 &              0.000014$\pm$0.000002 &              0.96849$\pm$0.00372 &   \second{0.02562$\pm$0.00217} &     \first{0.97418$\pm$0.00119} &              20.70674$\pm$3.56861 &     \first{0.10197$\pm$0.00222} \\
& TopoPCA           &              0.18857$\pm$0.00197 &   \second{0.36201$\pm$0.00186} &   \second{0.05296$\pm$0.00045} &   \second{0.00083$\pm$0.00002} &   \second{0.000009$\pm$0.000000} &              0.97030$\pm$0.00013 &              0.05584$\pm$0.00022 &              0.91790$\pm$0.00034 &              20.88881$\pm$0.29929 &              0.18315$\pm$0.00002 \\
& TopoAE &              0.11039$\pm$0.02948 &              0.39204$\pm$0.03264 &              0.05353$\pm$0.00959 &              0.00100$\pm$0.00015 &              0.000011$\pm$0.000002 &   \second{0.97998$\pm$0.00194} &              0.03156$\pm$0.00253 &              0.95612$\pm$0.00391 &              20.49122$\pm$0.93206 &   \second{0.12071$\pm$0.00238} \\

 \midrule
 %
 %
 \multirow{6}*{\data{MNIST}} & PCA               &              0.16754$\pm$0.00051 &              0.38876$\pm$0.00146 &              0.16301$\pm$0.00059 &              0.00160$\pm$0.00001 &   \second{0.000016$\pm$0.000000} &              0.90084$\pm$0.00016 &              0.16582$\pm$0.00022 &              0.74546$\pm$0.00048 &     \first{13.17437$\pm$0.00216} &              0.22269$\pm$0.00002 \\
& TSNE              &     \first{0.03767$\pm$0.00140} &     \first{0.27695$\pm$0.05266} &   \second{0.13266$\pm$0.02362} &              0.00214$\pm$0.00041 &              0.000024$\pm$0.000004 &              0.92101$\pm$0.00288 &     \first{0.03953$\pm$0.00129} &     \first{0.94624$\pm$0.00147} &              22.89261$\pm$0.24373 &                                -- \\
& UMAP              &   \second{0.07214$\pm$0.00091} &   \second{0.32063$\pm$0.00320} &              0.14568$\pm$0.00207 &              0.00234$\pm$0.00004 &              0.000027$\pm$0.000001 &     \first{0.93992$\pm$0.00066} &   \second{0.05109$\pm$0.00022} &   \second{0.93770$\pm$0.00039} &   \second{14.61535$\pm$0.04332} &                                -- \\
& Vanilla           &              0.44690$\pm$0.08540 &              0.61993$\pm$0.11742 &              0.15542$\pm$0.02203 &   \second{0.00156$\pm$0.00023} &              0.000016$\pm$0.000002 &              0.91293$\pm$0.00564 &              0.05828$\pm$0.00353 &              0.93699$\pm$0.00262 &              18.18105$\pm$0.21459 &     \first{0.13732$\pm$0.00160} \\
& TopoPCA           &              0.16138$\pm$0.00716 &              0.39157$\pm$0.01283 &              0.15556$\pm$0.00313 &              0.00162$\pm$0.00007 &              0.000017$\pm$0.000001 &              0.90301$\pm$0.00057 &              0.16297$\pm$0.00049 &              0.75040$\pm$0.00091 &              17.33353$\pm$0.82592 &              0.22477$\pm$0.00009 \\
& TopoAE &              0.32427$\pm$0.03312 &              0.34069$\pm$0.03056 &     \first{0.11012$\pm$0.01069} &     \first{0.00114$\pm$0.00010} &     \first{0.000012$\pm$0.000001} &   \second{0.93210$\pm$0.00132} &              0.05553$\pm$0.00044 &              0.92844$\pm$0.00142 &              19.57784$\pm$0.01812 &   \second{0.13884$\pm$0.00066} \\
 \midrule
 %
 %
\multirow{6}*{\data{CIFAR}} & PCA               &              0.27320$\pm$0.00014 &              0.59073$\pm$0.00004 &   \second{0.01961$\pm$0.00001} &     \first{0.00023$\pm$0.00000} &     \first{0.000002$\pm$0.000000} &     \first{0.93130$\pm$0.00000} &              0.11921$\pm$0.00005 &              0.82117$\pm$0.00002 &     \first{17.71567$\pm$0.00084} &              0.14816$\pm$0.00000 \\
& TSNE              &     \first{0.04451$\pm$0.00222} &              0.62733$\pm$0.01427 &              0.03014$\pm$0.00333 &              0.00073$\pm$0.00007 &              0.000009$\pm$0.000001 &              0.90300$\pm$0.00611 &     \first{0.10265$\pm$0.00242} &   \second{0.86325$\pm$0.00151} &   \second{25.61099$\pm$0.11551} &                                -- \\
& UMAP              &   \second{0.06934$\pm$0.00202} &              0.61673$\pm$0.00052 &              0.02562$\pm$0.00019 &              0.00050$\pm$0.00001 &              0.000006$\pm$0.000000 &              0.92045$\pm$0.00013 &              0.12680$\pm$0.00028 &              0.81668$\pm$0.00019 &              33.57785$\pm$0.00796 &                                -- \\
& Vanilla           &              0.37737$\pm$0.06507 &              0.66834$\pm$0.02992 &              0.03458$\pm$0.00448 &              0.00062$\pm$0.00021 &              0.000007$\pm$0.000002 &              0.85072$\pm$0.00429 &              0.13204$\pm$0.00316 &     \first{0.86359$\pm$0.00442} &              36.26827$\pm$0.56159 &   \second{0.14030$\pm$0.00190} \\
& TopoPCA           &              0.27207$\pm$0.00619 &   \second{0.58401$\pm$0.00515} &              0.01973$\pm$0.00040 &   \second{0.00024$\pm$0.00001} &   \second{0.000003$\pm$0.000000} &              0.92439$\pm$0.00115 &              0.12607$\pm$0.00133 &              0.81551$\pm$0.00139 &              30.73848$\pm$1.35231 &              0.15002$\pm$0.00076 \\
& TopoAE &              0.20877$\pm$0.00951 &     \first{0.55642$\pm$0.00412} &     \first{0.01879$\pm$0.00051} &              0.00031$\pm$0.00002 &              0.000003$\pm$0.000000 &   \second{0.92691$\pm$0.00100} &   \second{0.10809$\pm$0.00210} &              0.84514$\pm$0.00359 &              37.85914$\pm$0.03303 &     \first{0.13975$\pm$0.00171} \\
\bottomrule
\end{tabular}
}

    \caption{%
      Extended version of the table from the main paper, showing more
      length scales and variance estimates.
    }
    \label{tab:supp_results}
\end{sidewaystable*}

\end{document}